\documentclass[10pt,onecolumn]{article}
\usepackage{amsmath,amsfonts,bm}

\def\eqref#1{equation~\ref{#1}}

\def\1{\bm{1}}

\def\eps{{\epsilon}}
\def\diag{{\textnormal{diag}}}

\def\sspan{{\textnormal{span}}}

\def\rt{{\textnormal{t}}}

\def\vzero{{\bm{0}}}

\def\vb{{\bm{b}}}

\def\vq{{\bm{q}}}

\def\vv{{\bm{v}}}
\def\vw{{\bm{w}}}

\def\mB{{\bm{B}}}
\def\mC{{\bm{C}}}
\def\mD{{\bm{D}}}
\def\mE{{\bm{E}}}

\def\mI{{\bm{I}}}

\def\mQ{{\bm{Q}}}

\def\mS{{\bm{S}}}
\def\mT{{\bm{T}}}
\def\mU{{\bm{U}}}
\def\mV{{\bm{V}}}
\def\mW{{\bm{W}}}

\DeclareMathAlphabet{\mathsfit}{\encodingdefault}{\sfdefault}{m}{sl}
\SetMathAlphabet{\mathsfit}{bold}{\encodingdefault}{\sfdefault}{bx}{n}

\def\gA{{\mathcal{A}}}

\def\gE{{\mathcal{E}}}
\def\gF{{\mathcal{F}}}
\def\gG{{\mathcal{G}}}
\def\gH{{\mathcal{H}}}

\def\gL{{\mathcal{L}}}

\def\gR{{\mathcal{R}}}

\def\gV{{\mathcal{V}}}

\def\gX{{\mathcal{X}}}
\def\gY{{\mathcal{Y}}}

\def\sI{{\mathbb{I}}}

\def\sR{{\mathbb{R}}}

\newcommand{\E}{\mathbb{E}}

\newcommand{\R}{\mathbb{R}}

\newcommand{\Var}{\mathrm{Var}}

\newcommand{\Cov}{\mathrm{Cov}}

\DeclareMathOperator{\Tr}{Tr}

\usepackage{scalerel}

\newcommand{\paren}[1] {{\left ( #1 \right )}}
\newcommand{\brac}[1] {{\left [ #1 \right ]}}
\newcommand{\sparen}[2]{{\left ( #1 \; \middle \vert \; #2 \right )}}
\newcommand{\sbrac}[2]{{\left [ #1 \; \middle \vert \; #2 \right ]}}
\newcommand{\oset}[1] {{\left \{ #1 \right \}}}
\newcommand{\sset}[2]{{\left \{ #1 \; \middle \vert \; #2 \right \}}}
\newcommand{\dotp}[1]{{\left \langle #1 \right \rangle}}

\newcommand{\norm}[1] {{\left \| #1 \right \|}}

\def \ie {\textit{i.e.}}

\def \iid {\textit{i.i.d.}}
\def \wrt {\textit{w.r.t.}}

\def \err {\textnormal{err}}

\def \px {{P_\gX}}

\def \pa {{P_\gA}}

\def \lxp {{L^2(P_\gX)}}

\def \lap {{L^2(P_\gA)}}

\def \ka {{k_{A}^+}}
\def \kx {{k_{X}^+}}

\def \bPsi {\bar{\Psi}}
\def \tPsi {\tilde{\Psi}} 
\def \bPhi {\bar{\Phi}}
\def \tPhi {\tilde{\Phi}}
\def \tphi {\tilde{\phi}}

\def \tp {T_{P^+}}
\def \tpstar {T_{P^+}^{*}}
\def \tkx {T_{\kx}}
\def \tka {T_{\ka}}
\def \fp {\gF(P^+)}
\def \fep {\gF_{\epsilon}(P^+)}

\def \kp {k_{\Lambda}}

\def \dx {d_{\gX}}

\def \rt {\textnormal{RT}}
\def \tg {\textnormal{TG}}

\def \hpp {\gH_{P^+}}

\usepackage[utf8]{inputenc} 
\usepackage[T1]{fontenc}    
\usepackage[pdfencoding=unicode,psdextra,citecolor=blue,colorlinks=true]{hyperref}       
\usepackage{url}            
\usepackage{booktabs}       
\usepackage{amsfonts}       
\usepackage{amsthm}
\usepackage{amsmath}
\usepackage{amssymb}
\usepackage{nicefrac}       
\usepackage{microtype}      
\usepackage[table]{xcolor}         
\usepackage{algorithm}
\usepackage[noend]{algpseudocode}
\usepackage{graphicx}
\usepackage{wrapfig}
\usepackage{float, caption, subcaption}
\usepackage{mathtools}
\usepackage{enumitem}
\usepackage{mathrsfs}
\usepackage{cleveref}
\usepackage{verbatim}
\usepackage[most]{tcolorbox}
\usepackage{diagbox}
\usepackage{tikz}
\usepackage{pgfplots}  
\pgfplotsset{compat=1.18}  
\usetikzlibrary{calc, backgrounds}

\definecolor{todo}{RGB}{250,95,95}
\definecolor{edit}{RGB}{228,170,42}

\usepackage{colortbl}
\definecolor{Gray}{gray}{0.85}
\newcolumntype{a}{>{\columncolor{Gray}}l}
\newcolumntype{d}{>{\columncolor{Gray}}c}

\newtheorem{theorem}{Theorem}
\newtheorem{lemma}[theorem]{Lemma} 
\newtheorem{proposition}[theorem]{Proposition} 
\newtheorem{corollary}[theorem]{Corollary}
\newtheorem{definition}[theorem]{Definition}

\newtheorem*{remark}{Remark}

\crefname{equation}{eqn.}{eqns.}

\title{\huge Contextures: Representations from Contexts}

\author{ Runtian Zhai$^{1}$, Kai Yang$^{2}$, Che-Ping Tsai$^{1}$, Burak Var{\i}c{\i}$^{1}$, \\ 
Zico Kolter$^{1}$, Pradeep Ravikumar$^{1}$ \\
\\
$^{1}$Carnegie Mellon University \quad $^{2}$Peking University \\ 
\texttt{rzhai@alumni.cmu.edu, pradeepr@cs.cmu.edu} 
}

\begin{document}

\maketitle

\begin{abstract}

Despite the empirical success of foundation models, we do not have a systematic characterization of the \emph{representations} that these models learn. 
In this paper, we establish the contexture theory.
It shows that a large class of representation learning methods can be characterized as learning from the association between the input and a \emph{context variable}. Specifically, we show that many popular methods aim to approximate the top-$d$ singular functions of the expectation operator induced by the context, in which case we say that the representation \textit{learns the contexture}.
We demonstrate the generality of the contexture theory by proving that representation learning within various learning paradigms---supervised, self-supervised, and manifold learning---can all be studied from such a perspective.
We also prove that the representations that learn the contexture are optimal on those tasks that are compatible with the context.
One important implication of the contexture theory is that once the model is large enough to approximate the top singular functions, further scaling up the model size yields diminishing returns.
Therefore, scaling is not all we need, and further improvement requires better contexts.
To this end, we study how to evaluate the usefulness of a context without knowing the downstream tasks. We propose a metric and show by experiments that it correlates well with the actual performance of the encoder on many real datasets.
\end{abstract}

\section{Introduction}
Representation learning underpins the modern deep learning revolution, leading up to the remarkable recent successes of foundation models \cite{bommasani2021opportunities}.
But a critical question that has remained unanswered to a satisfactory extent is: why are these models learning anything useful, or perhaps even what representations are these models learning? Unlike classical statistical learning theory, where there is no mystery regarding the statistical estimand, the very target itself is unclear in representation learning.
For example, what are the representations that BERT \cite{devlin2018bert}---trained to do cloze tests---is learning, and why are they useful in understanding the sentiment of user reviews on Netflix? What representations do deep neural networks learn, and why are they useful if they cause neural collapse \cite{papyan2020prevalence}, where deep representations could collapse to a few clusters? Do the many different self-supervised learning methods~\cite{Balestriero2023ACO} all learn similar or disparate representations?

The responses to these questions are often muddled.
Many analyses are conflated with the mystery of deep learning generalization---the ability of large neural networks to learn function approximations that generalize to unseen points.
However, this is a very different problem from the mechanism of representation learning.
Our focus is on what representation learning (also known as ``pretraining'' in the context of foundation models) aims to capture, and why it can be applied to tasks completely different from the objectives used to train the representations. 
Another way this question is muddled is by recourse to scaling.
A popular viewpoint argues that even if the encoder performs poorly on one task, increasing the model size while keeping everything else the same could allow better performance to ``emerge'' \cite{wei2022emergent}.
However, substantial evidence suggests that certain abilities cannot emerge solely from scaling.
Additional training signals, such as alignment \cite{ouyang2022training}, are necessary.

The above questions are naturally interesting to learning theorists, but why should the broader machine learning community care about understanding the mechanism of representation learning, if empirical success seems to be always achievable with existing approaches by scaling up the model size, an empirical observation known as \textit{scaling laws} \cite{kaplan2020scaling}?
This is because sustainable success or progress is not always guaranteed. Ilya Sutskever recently remarked that ``pretraining as we know it will end'' \cite{ilya2024talk}, largely because the current pretraining paradigm is producing diminishing returns.
Understanding what representations are learned by foundation models is crucial for designing future generations of pretraining methods, and this is how this field can make scientific progress.

In this work, we establish \textbf{the contexture theory}, which provides a unified lens for inspecting a large class of representation learning methods. 
The central argument of the contexture theory is that representations are learned from the association between the input $X$ and a context variable $A$.
This framework is general enough to encompass a wide variety of learning paradigms, as we demonstrate in \Cref{sec:objectives}.

Now, suppose we are given a context variable $A$ along with $X$, how should we learn the representation?
In \Cref{sec:results} we prove that the optimal method is to approximate the top singular functions of the expectation operator induced by the context, in which case we say that the encoder ``\textbf{learns the contexture}''.
Such an encoder is optimal as long as the task is known to be \textbf{compatible} with the context, and in \Cref{sec:results} we define a quantitative measurement of such compatibility.

Our theory implies that the main consequence of making the model large is that the learned representation space will be brought closer to the space spanned by the top-$d$ singular functions of the expectation operator, which we empirically verify in \Cref{sec:scaling-law}.
Once the two spaces are close enough, further scaling will yield diminishing returns.
We envision that future breakthroughs in pretraining require \textit{context scaling}, where better contexts are learned from data, instead of heuristics.

Taking a first step towards context scaling, in \Cref{sec:context-evaluation} we study how to evaluate contexts.
This is a prerequisite because if we cannot even determine which contexts are good, then there will be no way to create better contexts.
The key takeaway from our analysis is that for a context to be useful (meaning that it can lead to good representations), the \textit{association} between $X$ and $A$ should be moderate---neither too strong nor too weak.
For example, if $A = X$, then their association is the strongest; if $A$ is independent of $X$, then their association is the weakest. However, neither context is useful because $A$ does not provide additional information about $X$.
In \Cref{sec:metric}, we propose a quantitative measurement of context usefulness that can be efficiently estimated and does not require knowledge of the downstream task.
We also demonstrate empirically that the proposed metric correlates with the performance of the encoder on many real datasets.

In one sentence, the key contribution of this work is clarity on the target of a large class of representation learning methods---the singular functions of the expectation operator. However, we do not discuss the numerical aspect of approximating these functions, which requires an expressive model architecture and a good optimizer, and such analyses are left to future work.

\subsection{Examples of Representation Learning Methods}
\label{sec:method-examples}

Supervised learning is the simplest way to learn representations.
For example, the representations of neural networks pretrained on ImageNet~\cite{russakovsky2015imagenet} were very popular in the early days of the deep learning boom~\cite{huh2016makes}.
Specifically, one uses the output of an intermediate layer, typically the one before the last linear layer, as the representation. However, it has never been fully explained why the penultimate layer works so well on tasks very different from the pretraining task.

Self-supervised learning (SSL) is currently the most common way of learning representations.
Two of the most widely used SSL categories are multi-view learning and denoising autoencoders. Multi-view learning includes contrastive learning \cite{oord2018representation,chen2020simple} and non-contrastive learning \cite{NEURIPS2020_f3ada80d,zbontar2021barlow,bardes2021vicreg}.
Denoising autoencoders have wide applications, including language \cite{devlin2018bert,radford2019language}, vision \cite{he2022masked}, videos \cite{gupta2023siamese}, and more.

Manifold learning is a classical method that aims to capture the geometry of the data.
Examples such as locally linear embedding (LLE) \cite{roweis2000nonlinear} and Laplacian eigenmaps \cite{belkin2003laplacian} formulate manifold learning as node representation learning on a graph, where connected nodes should have similar embeddings.

\subsection{Related Work}
\label{sec:related-work}

Understanding what representations an encoder learns has long been a hot research topic in machine learning.
In the pre-deep-learning-boom era, early work on manifold learning revealed the connection between representation learning and approximating the top eigenfunctions of a kernel \cite{belkin2003laplacian,bengio2004learning,coifman2006diffusion}.
Moreover, using the eigenvectors of the graph Laplacian as node representations on graphs was a classical technique in graph applications \cite{belkin2002using,zhu2003semi}.

Although deep learning has achieved phenomenal success, it is not as theoretically driven as traditional methods, which makes it difficult to analyze what representations it learns.
Most early work experimentally studied the representations of neural networks with visualization tools \cite{DBLP:journals/corr/SimonyanVZ13,zeiler2014visualizing}, or data manipulation \cite{zhang2017understanding}.
These experimental tools are still widely used today to analyze large language models \cite{yin-neubig-2022-interpreting,nanda2023progress,nanda-etal-2023-emergent}.

On the theoretical understanding of representations, two lines of work are closely related to this paper.
The first line studies representation alignment \cite{pmlr-v97-kornblith19a,pmlr-v235-huh24a,fumero2024latent}.
These papers mainly focus on comparing between two representations, while this work aims to evaluate a single representation and the context in which it is trained.
Representation similarity has also been studied in neuroscience \cite{kriegeskorte2008representational}.
The second line develops the spectral theory of self-supervised learning (SSL).
SSL has achieved remarkable success in recent years \cite{radford2019language,chen2020simple,zbontar2021barlow,bardes2021vicreg,he2022masked,baevski2022data2vec,oquab2023dinov2,assran2023self}.
See the SSL cookbook by \cite{Balestriero2023ACO} for a summary of these methods.
\cite{haochen2021provable,johnson2022contrastive} related contrastive learning to the spectrum of the augmentation graph and the positive-pair kernel.
Then, \cite{zhai2023understanding} extended the spectral theory to all SSL methods, not just contrastive learning.
This paper extends the spectral theory to an even broader scope than \cite{zhai2023understanding}.
Other theoretical work on SSL studies its training dynamics \cite{damian2022neural,jing2022understanding,tian2022understanding} and builds its connection to information theory \cite{achille2018emergence,balestriero2022contrastive,shwartz2023information}.

Another line of work for characterizing the representations aims to learn disentangled or causally associated representations~\cite{higgins2018towards,scholkopf2021toward}. It is shown that such representations can be provably recovered, provided sufficient variables are present in the environments, either through auxiliary variables or interventions~\cite{khemakhem20a,varici2024general,buchholz2023learning,yao2024multi}. Some of these results further require stringent parametric assumptions.

Prior work has developed quite different theoretical frameworks for different pretraining methods, and each framework is not applicable to other settings.
In this work, we provide a unified framework that spans the works discussed above and more---manifold learning, SSL based on augmentation, supervised learning, graph representation learning, etc. 
It further extends the spectral theory in \cite{zhai2023understanding} to other paradigms beyond SSL.

\section{Definitions and Examples}
Let $\gX$ be the \textit{input space}.
The goal of representation learning is to learn a feature encoder $\Phi : \gX \rightarrow \R^d$. We call $\Phi(x)$ the \textit{embedding} of $x$, and $d$ the embedding dimension.
Let $\px$ be the data distribution.
Throughout this work, we assume $\px$ to be fixed.

The central argument of the contexture theory is that representations are learned from the association between the input $X \in \gX$ and a context variable $A \in \gA$. 
$\gA$ is called the context space.
Let $P^+(x,a)$ be the joint distribution of $X$ and $A$, with marginal distributions $\px$ and $\pa$.
Let $\lxp$ be the $L^2$ function space \wrt{} $\px$, with inner product $\dotp{f_1, f_2}_\px = \E_{X \sim \px}[ f_1(X) f_2(X) ]$ and norm $\norm{f}_\px = \sqrt{\dotp{f,f}_\px}$.
Define $\lap, \dotp{\cdot, \cdot}_{\pa}, \norm{\cdot}_{\pa}$ for $\pa$ similarly.

\subsection{Examples of Contexts}
\label{sec:examples}
We first introduce some commonly used contexts, their corresponding $P^+$, and how $P^+$ is provided in practice.

\begin{enumerate}[itemsep=0pt, topsep=0.5pt,leftmargin=*]
    \item \textbf{Labels} are a common type of context in machine learning. They can take different forms, such as discrete categories in classification, continuous values in regression, or structured outputs like text sequences in image captioning. Labels may be obtained from human annotators or in pseudo-forms, such as clusters or teacher models. Typically, labels are provided as compatible pairs sampled from the joint distribution $P^+(x,a)$.\looseness=-1
    
    \item \textbf{Random transformations} generate different views of the same data point. Common transformations include adding random noise to inputs, as seen in diffusion models and denoising autoencoders, or randomly corrupting/masking inputs, as in SimCLR and masked autoencoder. These transformations are typically defined by domain experts and are specified through a predefined conditional distribution $P^+\sparen{a}{x}$.

    \item \textbf{Graphs} provide locality information about the inputs. The edge values are typically given, which represent the similarity between two inputs. 
    We can have a continuous space version of such a graph kernel to capture local geometry, which in the limit can be related to differential operators such as the Laplace-Beltrami operator on a manifold. In this case, we have $\gA = \gX$ with the conditional distribution $P^+\sparen{a}{x}$ proportional to the edge values between $x$ and $a$. Please refer to \Cref{sec:graph} for a more detailed construction of $P^+\sparen{a}{x}$.
    
    \item \textbf{Stochastic Features} are potentially stochastic predefined or pretrained mappings from $\gX$ to a vector space, which we denote by $a = \Omega(x)$.
    This encompasses teacher models that provide stochastic pretrained feature encoders. In contrast to previous instances, here $P^+\sparen{a}{x}$ is directly described via a reparameterization or structural equation for $a$ in terms of $x$.
\end{enumerate}

\subsection{Induced Kernels and Expectation Operator}

The joint distribution $P^+$ fully determines the association between $X$ and $A$.
It induces the positive-pair kernel~\cite{johnson2022contrastive} and the dual kernel \cite{zhai2023understanding} defined as follows.

\begin{definition}
\label{def:two-kernels}
    The \textbf{positive-pair kernel} 
    $\ka$ and its \textbf{dual kernel} $\kx$ are defined as
    \begin{align*}
        \ka(a,a') & = \frac{P^+(a,a')}{\pa(a) \pa(a')} 
        = \frac{\int P^+\sparen{a}{x} P^+\sparen{a'}{x} d \px (x)}{\pa(a) \pa(a')}; \\ 
     \kx(x,x') & = \frac{P^+(x,x')}{\px(x) \px(x')}  = \frac{\int P^+\sparen{x}{a} P^+ \sparen{x'}{a} d \pa(a)}{\px(x) \px(x')}. 
    \end{align*}
\end{definition}
These two kernels are density ratios between joints and marginals for $A$ and $X$, respectively. The kernels captures how more likely 
$(a,a')$ or $(x,x')$ appear together than independently given $P^+$.
$P^+$ also induces the following expectation operator. Intuitively, given a function $g \in \lap$, the expectation operator computes the expectation of $g(A)$ conditioned on any given $x$.
\begin{definition}
  The \textbf{expectation operator} $\tp: \lap \rightarrow \lxp$ is defined as for all $g \in \lap$,
\begin{equation*}
        \paren{\tp g} (x) = \int g(a) P^+\sparen{a}{x} da = \E \sbrac{g(A)}{x}. 
\end{equation*}
Its adjoint operator $\tpstar : \lxp \rightarrow \lap$, which satisfies $\langle f, \tp g \rangle_{\px} = \langle \tpstar f, g \rangle_{\pa} \; (\forall f \in \lxp, g \in \lap)$, is given by $\paren{\tpstar f}(a) = \int f(x) \frac{P^+\sparen{a}{x} \px(x)}{\pa(a)} dx = \E \sbrac{f(X)}{a}$.
\end{definition}

Now we discuss the spectral properties of these operators.
Define the kernel integral operator $\tka: \lap \rightarrow \lap$ as $(\tka g)(a) = \int g(a') \ka(a,a') d \pa (a')$.
Define the other operator $\tkx: \lxp \rightarrow \lxp$ similarly.
It is easy to see that $\tka = \tpstar \tp$, and $\tkx = \tp \tpstar$.

We call $\lambda \in \R$ an eigenvalue of $\tka$ with eigenfunction $\nu \in \lap$, if $\tka \nu = \lambda \nu$. 
Suppose $\tka$ is a Hilbert-Schmidt integral operator. Then, we can order its eigenvalues by $1 = \lambda_0 \ge \lambda_1 \ge \cdots \ge 0$, and the corresponding eigenfunctions $\nu_0,\nu_1,\cdots$ form an orthonormal basis (ONB) of $\lap$.
Here $\lambda_i \le 1$ because of Jensen's inequality,
and $\nu_0 \equiv 1$ is always an eigenfunction of $\tka$ with $\lambda_0 = 1$.
Similarly, we can order the eigenvalues of $\tkx$ by $1 = \kappa_0 \ge \kappa_1 \ge \cdots \ge 0$, and the eigenfunctions $\mu_0, \mu_1,\cdots$ form an ONB of $\lxp$, where $\mu_0 \equiv 1$.
We also have the following result.
\begin{lemma}[Duality property, \cite{zhai2023understanding}, Proposition~1]
\label{lem:duality}
    For all $i$, we have $\lambda_i = \kappa_i \in [0,1]$. And if $\lambda_i > 0$, then we have $\mu_i = \lambda_i^{-\frac{1}{2}} \tp \nu_i$, and $\nu_i = \lambda_i^{-\frac{1}{2}} \tpstar \mu_i$.
\end{lemma}
We call $s_i = \lambda_i^{\frac{1}{2}}$ a \textbf{singular value} of $\tp$, associated with left \textbf{singular function} $\mu_i \in \lxp$ and right singular function $\nu_i \in \lap$.
Since we choose $\mu_0 \equiv 1$ and $\nu_0 \equiv 1$, all other $\mu_i$ (and $\nu_i$) must have zero mean because they are orthogonal to $\mu_0$ (and $\nu_0$).
Using these singular functions, we can spectrally decompose $P^+$ as follows.
\begin{lemma}
\label{lem:spectral-decomposition} The spectral decomposition of $P+$ is
    $P^+(x,a) = \sum_i s_i \mu_i(x) \nu_i(a) \px(x) \pa (a)$.
\end{lemma}
\begin{proof}
   $\forall i, \;   
   \dotp{\frac{P^+(x,a)}{\px(x) \pa(a)}, \nu_i }_{\pa} = \int P^+\sparen{a}{x} \nu_i(a) da = (\tp \nu_i) (x) = 
   \paren{\lambda_i^{\frac{1}{2}} \mu_i } (x) = s_i \mu_i(x)$.
    Since $(\nu_i)_{i \ge 0}$ is an ONB, we have $\frac{P^+(x,a)}{\px(x) \pa(a)} = \sum_{i=0}^{\infty} s_i \mu_i(x) \nu_i(a)$.
\end{proof}

The first key result of the contexture theory is that the optimal $d$-dimensional representation should recover the linear space spanned by the top-$d$ singular functions $\mu_1,\cdots,\mu_d$.
We say that such a representation learns the contexture.
Note that the constant function $\mu_0 \equiv 1$ is excluded, as it does not need to be learned—there is no benefit in allocating a dimension to encode something already universally present.

\begin{definition}
\label{def:learn-contexture}
    A $d$-dimensional encoder $\Phi = [\phi_1,\cdots,\phi_d]$ \textbf{learns the contexture} of $P^+$, if there exists a set of top-$d$ singular functions $\oset{\mu_1,\cdots,\mu_d}$ of $\tp$, such that $\sspan \oset{\phi_1,\cdots,\phi_d} =  \sspan \oset{\mu_1,\cdots,\mu_d}$. We also say that $\Phi$ \textbf{extracts the top-$d$ eigenspace} of $\tkx$. 
\end{definition}

If the multiplicity of $s_d$ is more than $1$, then $\Phi$ recovering the span of any top-$d$ singular functions suffices.
The intuition why learning the contexture is ideal is that such a representation keeps the most information (variance) of the context, which is analogous to principal component analysis (PCA) in the finite-dimensional case.
Consider the case where $\gX$ and $\gA$ are both finite sets.
Let $N = |\gX|$ and $M = |\gA|$.
Then, a function $f \in \lxp$ is a vector in $\R^N$, $g \in \lap$ is a vector in $\R^M$, and $\tp$ is essentially a matrix $\mT \in \R^{N \times M}$.
Suppose we want to learn a $d$-dimensional embedding $\mE \in \R^{N \times d}$ for the $N$ samples in $\gX$, and it should preserve the information of $\mT$ as much as possible, then what should we do?
PCA states that we should use the top-$d$ left singular vectors of $\mT$ as $\mE$, which are equivalent to the top-$d$ eigenvectors of $\mT \mT^{\top}$, because they maximize the explained variance.
Similarly, functional spaces are essentially infinite-dimensional vector spaces, so the $d$-dimensional embedding of $X$ that preserves the most information of $\tp$ consists of the top-$d$ left singular functions of $\tp$, or equivalently the top-$d$ eigenfunctions of $\tkx = \tp \tpstar$.

\section{Learning the Contexture}
\label{sec:objectives}

In this section, we show that every example method in \Cref{sec:method-examples} does one of the following:
\begin{enumerate}[label=(\roman*)]
    \item Extracts the top-$d$ eigenspace of $\tkx = \tp \tpstar$ (learns the contexture of $P^+$), which according to \Cref{def:learn-contexture} is equivalent to recovering the span of $\mu_1,\cdots,\mu_d$ (excluding $\mu_0 \equiv 1$);
    \item Extracts the top-$d$ eigenspace of $\tp \Lambda \tpstar$, where $\Lambda$ is the integral operator of a kernel $k_{\Lambda}(a,a')$, that is $(\Lambda g)(a) = \int g(a') k_{\Lambda}(a,a') d \pa(a')$. $k_\Lambda$ is called the \textbf{loss kernel}, which is defined by the loss function used in the objective.
    Since the constant function is not necessarily the top eigenfunction of $\tp \Lambda \tpstar$, in this case, we do not exclude any eigenfunction.
\end{enumerate}

\textbf{Notation:} For any $f \in \lxp$, denote its mean by $\bar{f} = \E_{\px}[f(X)]$, and its centered version by $\tilde{f} = f - \bar{f}$. The same notation is used for multi-dimensional functions and random variables, as long as the distribution is clear from context.
The covariance matrix of any $\Phi: \gX \rightarrow \R^d$, denoted by $\Cov_{\px}[\Phi]$, is a $d \times d$ matrix $\mC$ where $\mC[i,j] = \dotp{\tphi_i, \tphi_j}_{\px}$.

\subsection{Supervised Learning: Label Context}
In the supervised learning paradigm, the context variable $A$ is the label of $x$, and $P^+\sparen{\cdot}{x}$ is the label distribution of $x$.
Consider minimizing the mean squared error (MSE):
\begin{equation}
\label{eqn:supervised-obj}
    \gR(\Phi) = \min_{\mW \in \R^{d_A \times d}, \vb \in R^{d_A}} \E_{(X,A) \sim P^+} \brac{\norm{\mW \Phi(X) + \vb - A}_2^2} .
\end{equation}
That is, $\Phi(X)$ is the output of the layer before the last linear layer in a neural network, and $\vb$ denotes the bias.
If $\vb$ can be an arbitrary vector, then the linear layer is biased; if $\vb = \vzero$ is fixed, then the linear layer is unbiased.

First, we study classification tasks where $A$ is one-hot. 
\begin{theorem}[Proof in \Cref{app:proof-thm-supervised-cls}]
\label{thm:thm-supervised-cls}
    Let $A$ be a one-hot random vector.
    Suppose the linear layer is unbiased, that is $\vb = \vzero$.
    Then, $\Phi^*$ minimizes $\gR(\Phi)$ if and only if it extracts the top-$d$ eigenspace of $ \tp \Lambda \tpstar$, where $\kp(a,a') = \sI [a=a']$, or $(\Lambda g)(a) = g(a) \pa(a)$.
    If all classes have the same size, then the top-$d$ eigenfunctions of $ \tp \Lambda \tpstar$ and $\tp \tpstar$ are the same.
\end{theorem}

This theorem works for randomized labels, where each $x$ could belong to multiple classes with certain probabilities.
The loss kernel $\kp$ is a consequence of class imbalance, and it puts more weights on larger classes.
Indeed, in practice, smaller classes are harder to learn.
To get rid of $\Lambda$, we can use the class-balanced risk (also known as importance weighting \cite{shimodaira2000improving}):
\begin{equation*}
    \gR_{\rm bal}(\Phi) = \min_{\mW \in \R^{d_A \times d}, \vb \in R^{d_A}} \E_{(X,A) \sim P^+} \brac{\frac{\norm{\mW \Phi(X) + \vb - A}_2^2}{\sqrt{\pa(A)}}} .
\end{equation*}
\begin{theorem}[Proof in \Cref{app:proof-thm-supervised-bal-cls}]
\label{thm:thm-supervised-bal-cls}
    Under the setting of \Cref{thm:thm-supervised-cls}, suppose the linear layer is biased. Then, $\Phi^*$ minimizes $\gR_{\rm bal}(\Phi)$ if and only if it learns the contexture of $P^+$.
\end{theorem}
Interestingly, our result can partially explain \textbf{neural collapse}.
\cite{papyan2020prevalence} empirically showed that when there are $d$ classes of equal sizes and the label $A$ is deterministic, a sufficiently trained deep representation will collapse to an equiangular tight frame (ETF) $\phi_1,\cdots,\phi_d$, which are defined as $\phi_i(x) = c(\sI[x \text{ belongs to class } i] - d^{-1})$ for all $i \in [d]$ and some constant $c$.
The span of $\phi_1,\cdots,\phi_d$ is the same as the span of the top-$d$ eigenfunctions of $\tp \Lambda \tpstar$.
However, our result cannot explain why $\phi_1,\cdots,\phi_d$ converge to the exact functions as above---it only proves that they will span the same space.
To explain this, one needs to analyze the training dynamics, which depend on the specific optimizer, such as gradient-based methods, while our results are independent of the optimizer.

When the classes have different sizes, it is easy to see that the dual kernel of $\tp \Lambda \tpstar$ is $\kx(x,x') = \sI[x \text{ and } x' \text{ have the same label}]$. This is equivalent to the simplex-encoded labels interpolation (SELI) defined by \cite[Definition~2]{thrampoulidis2022imbalance}, which generalizes neural collapse.

For regression where $A$ is an arbitrary Euclidean vector, using the same objective as \Cref{eqn:supervised-obj}, we can prove the following result.
\begin{theorem}[Proof in \Cref{app:proof-thm-obj-regression}]
\label{thm:obj-regression}
$\Phi^*$ minimizes \Cref{eqn:supervised-obj} if and only if $\Phi^*$ extracts the top-$d$ eigenspace of $\tp \Lambda \tpstar$.
If the linear layer is unbiased ($\vb = \vzero$), then $\kp(a,a') = \dotp{a, a'}$;
if it is biased ($\vb$ can be arbitrary), then $\kp(a,a') = \dotp{\tilde{a}, \tilde{a'}}$.
\end{theorem}
\begin{remark}
    Kernel $\kp(a,a') = \dotp{a,a'}$ is called the \textbf{linear kernel} on $\gA$, and $\kp(a,a') = \dotp{\tilde{a},\tilde{a'}}$ is called the \textbf{centered linear kernel} \wrt{} $\pa$. \Cref{thm:thm-supervised-cls} is a special case of \Cref{thm:obj-regression}.
\end{remark}

\subsection{Self-supervised Learning (SSL): Transformation Context}

Two major types of SSL are multi-view learning and denoising autoencoders.
Multi-view learning independently samples two views $A, A^+$ from $P^+(\cdot | x)$ for every $x$. $A^+$ is called a positive sample of $A$.
Then, one trains $\Psi: \gA \to \R^d$ such that $\Psi(A) \approx \Psi(A^+)$.
This $\Psi$ is an encoder on $\gA$ instead of $\gX$, so at downstream we need to convert $\Psi(a)$ to $\Phi(x)$, which is typically done via the \textbf{average encoder}:
\begin{equation*}
    \Phi(x) = (\tp \Psi)(x) = \int \Psi(a) dP^+\sparen{a}{x}.
\end{equation*}
By \Cref{lem:duality}, we have the following corollary.
\begin{corollary}
    Let $s_d > 0$. The average encoder $\Phi$ spanning the span of the left top-$d$ singular functions of $\tp$ is equivalent to $\Psi$ spanning the span of the right top-$d$ singular functions of $\tp$.
\end{corollary}
Enforcing $\Psi(A) \approx \Psi(A^+)$ alone leads to the degenerate solution where $\Psi$ gives the same embedding to all $a$.
This is called the \textit{feature collapse} problem. 
There are two solutions: contrastive learning and non-contrastive learning.
Prior work by \cite{haochen2021provable,johnson2022contrastive,zhai2023understanding} showed that the following spectral contrastive loss $\gL_{\rm{C}}$ and spectral non-contrastive loss $\gL_{\rm{N}}$ can learn the contexture of $P^+$. 
Let $A^+$ be a positive sample of $A$ drawn from the same $x$, and $A^-$ be a negative sample drawn from another $x$ independently.
\begin{align*}
    \gL_{\rm{C}} & = \E \brac{- \dotp{ \tPsi(A), \tPsi(A^+)} + \frac{1}{2} \dotp{\tPsi(A), \tPsi(A^-) } ^2 }  ; \\ 
    \gL_{\rm{N}} & = \E \brac{\norm{\Psi(A) - \Psi(A^+)}_2^2} \; \; \text{s.t.} \; \; \Cov_{\pa}\brac{\Psi} = \mI  ,
\end{align*}
where the $(i,j)$-th entry of $\Cov_{\pa}\brac{\Psi}$ is $\dotp{\psi_i, \psi_j}_{\pa}$, $i, j \in [d]$.
Minimizing $\gL_{\rm{N}}$ is a constrained optimization problem.
The constraint $\Cov_{\pa}\brac{\Psi} = \mI$ is called the \textbf{orthonormality constraint}.
It makes sure that $\Psi$ must be rank-$d$, so that it cannot be a constant function on $\gA$.

\begin{theorem}[Proof in \Cref{app:proof-thm-ssl-singular}]
\label{thm:ssl-singular}
$\Psi^*$ minimizes $\gL_{\rm{C}}$ or $\gL_{\rm{N}}$ if and only if $\tPhi^* = \tp \tPsi^*$ learns the contexture.
\end{theorem}

For denoising autoencoders, suppose $\gX \subseteq \R^{d_X}$.
Then, consider minimizing the following objective:

\begin{equation}
\label{eqn:obj-reconstruction}
\gR(\Psi) = \min_{\mW \in \R^{d_X \times d} , \; \vb \in \R^{d_X}} \; \underset{(X,A) \sim P^+}{\E} \; \brac{ \norm{ \mW \Psi(A) + \vb - X }_2^2 }  .
\end{equation}
\begin{theorem}
\label{thm:obj-reconstruction}
Let $\Psi^*$ be any minimizer of \Cref{eqn:obj-reconstruction}. Then, $\tPsi^*$ extracts the top-$d$ eigenspace of $\tpstar \Lambda \tp$, where $\Lambda$ is the integral operator of $\kp(x,x') = \dotp{\tilde{x}, \tilde{x}'}$ if $\vb$ can be an arbitrary vector, or $\kp(x,x') = \dotp{x,x'}$ if $\vb = \vzero$. Consequently, $\tPhi^* = \tp \tPsi^*$ extracts the top-$d$ eigenspace of $\tp \tpstar \Lambda$.
\end{theorem}
\begin{proof}
    The proof is the same as \Cref{thm:obj-regression}.
\end{proof}

\subsection{Node Representation Learning: Graph Context}
\label{sec:graph}
Let $\gG = (\gV, \gE)$ be an undirected graph, where edge $(u,v)$ has a weight $w(u,v)$ such that $w(u,v) = w(v,u) \ge 0$.
Let the degree of node $u$ be $d(u) = \sum_{v \in \gV} w(u,v)$, and $d_{\textrm{sum}} = \sum_v d(v)$. 
Let $\px(u) = \frac{d(u)}{d_{\textrm{sum}}}$ be a distribution on $\gV$, and $P_w(u,v) = \frac{w(u,v)}{d_{\textrm{sum}}}$ be a distribution on $\gE$.
Define $P^+(v|u) = \frac{w(u,v)}{d(u)}$.
Then, the following optimization problem with a similar orthonormality constraint can learn the contexture of $P^+$.
\begin{equation}
\label{eqn:node-repre-obj}
    \underset{\Phi: \gX \rightarrow \R^d}{\text{minimize}} \quad  \frac{1}{2} \E_{(u,v) \sim P_w} \brac{ \norm{\Phi(u) - \Phi(v)}_2^2 } \qquad \text{s.t.} \qquad  \Cov_{\px}[\Phi] = \mI  .
\end{equation}
\begin{theorem}[Proof in \Cref{app:proof-thm-node-repre}]
\label{thm:node-repre}
    Let $\Phi^*$ be any solution to \Cref{eqn:node-repre-obj} (so that for any constant $c$, $\Phi^* + c$ is also a solution).
    Then, $\tPhi^*$ learns the contexture of $P^+$.
\end{theorem}

\section{Optimality of the Contexture}
\label{sec:results}

So far, we have shown that commonly used learning objectives can learn the contexture.
Now, we present a rigorous theory that addresses the question of  \emph{why} and \emph{when} learning the contexture is optimal. Arguably, no representations can be good for all downstream tasks. However, we show that a feature encoder that learns the contexture is optimal for the class of tasks that are \textit{compatible} with the context. This provides a quantitative characterization of when a task respects the human prior knowledge the context incorporates.
Interestingly, as we detail in the sequel, this has intriguing implications for scaling laws.

\subsection{Compatibility}

The ultimate evaluation of an encoder is its performance on relevant downstream tasks.
Most downstream tasks, such as prediction, clustering, and segmentation, can be associated with a target function $f^* \in \lxp$.
For example, multi-class classification can be associated with multiple one-vs-all labeling functions.
Moreover, if we are fitting a linear predictor on top of $\Phi$, then the mean and variance of $f^*$ do not matter because we can change the weight and bias of the predictor accordingly.
Thus, we can assume that $f^*$ is normalized, that is, it has zero mean and unit variance.

We say that a context $P^+$ is \textbf{useful} for a task $f^*$, if it can help us learn a good predictor for $f^*$.
Formally, suppose $A \sim P^+ \sparen{\cdot}{x}$ is a random corruption of $x$.
Consider the scenario where we have a corrupted training set $\oset{(a_i,y_i)}$, where $a_i \sim P^+\sparen{\cdot}{x_i}$ and $y_i = f^*(x_i)$.
That is, we cannot see the original samples $x_i$, but can only see the corrupted samples $a_i$.
To learn a predictor on this training set, we can train a predictor $\hat{g} : \gA \rightarrow \gY$, and then use $\hat{f} = \E \sbrac{\hat{g}(A)}{x}$.
At test time, given input $x$, we can draw $a \sim P^+\sparen{\cdot}{x}$ and then output the average of $g^*(a)$.
For this procedure to work, two conditions are necessary:
\begin{itemize}[itemsep=0pt]
    \item There exists $g^* \in \lap$ s.t. $f^*(x) = \E\sbrac{g^*(A)}{x}$.
    \item This $g^*$ has a low $\Var \sbrac{g^*(A)}{x}$ on average over $x$.
\end{itemize}
If $\Var \sbrac{g^*(A)}{x}$ is high, then $g^*(a)$ will be far away from $y=f^*(x)$, so fitting $\hat{g}$ on $(a, y)$ will not work.
Based on these insights, we define compatibility as follows.
\begin{definition}
\label{def:compatibility}
    The \textbf{compatibility} with $P^+$ of any non-zero $f \in \lxp$ is defined as
    \begin{equation*}
        \rho(f, P^+) = \max_{g \in \lap, g \neq \vzero} \frac{\dotp{\tilde{f}, \tp g}_{\px}}{\norm{\tilde{f}}_{\px} \norm{g}_{\pa}} \in [0,1]  .
    \end{equation*}
\end{definition}
For further insight, let $f = \sum_i u_i \mu_i$ and $g = \sum_i v_i \nu_i$. 
Then, $\rho(f,P^+) = \underset{v_i}{\max} \frac{\sum_{i \ge 1} s_i u_i v_i}{\sqrt{\sum_{i \ge 1} u_i^2} \sqrt{\sum_i v_i^2}} = \sqrt{\frac{\sum_{i \ge 1} s_i^2 u_i^2}{\sum_{i \ge 1} u_i^2}}$ by Cauchy-Schwarz inequality (the optimal $v_i$ satisfy $v_0 = 0$ and $v_i \propto s_i u_i$ for $i \ge 1$).
For any $\epsilon > 0$, we define the class of $(1-\epsilon)$-compatible tasks as
\begin{equation*}
    \fep = \oset{f \in \lxp: \E[f] = 0, \rho(f, P^+) \ge 1 - \epsilon}.
\end{equation*}
This class of tasks satisfies the two conditions, \ie{} we can find a $g^*$ with low variance $\Var \sbrac{g^*(A)}{x}$:
\begin{theorem}[Proof in \Cref{app:proof-desiderata}]
\label{thm:desiderata}
For any $f^* \in \fep$, there exists a $g^* \in \lap$ such that $f^*(x) = \E\sbrac{g^*(A)}{x}$, and $g^*$ satisfies
\begin{equation*}
\underset{X \sim \px}{\E} \; \underset{A,A' \sim P^+\sparen{\cdot}{X}}{\E} \brac{ \paren{g^*(A) - g^*(A')}^2 } \le 4\epsilon \norm{g^*}_{\pa}^2  .
\end{equation*}
\end{theorem}
Now that we have a class of tasks compatible with $P^+$, we evaluate $\Phi$ by its worst-case approximation error on $\fep$.
The most common way to evaluate $\Phi$ is to fit a linear predictor on top, also called a \textbf{linear probe}, which is the focus of our attention (other methods for using $\Phi$ include fitting a small neural network on top, using a kernel method, or using KNN). 
Specifically, the worst-case approximation error of $\Phi$ on $\gF \subset \lxp$ is the maximum error of the optimal linear probe in estimating any function in $\gF$. In this work, we focus on the $L_2$ error.
\begin{definition}
\label{def:worst-case-app}
Let $\gF \subset \lxp$ be a function class where $f \in \gF \Rightarrow \alpha f \in \gF$ for all $\alpha \in \R$.
The \textbf{worst-case approximation error} of $\Phi: \gX \rightarrow \R^d$ on $\gF$ is defined as
\begin{equation*}
    \err \paren{\Phi; \gF}  = \max_{f \in \fp, \; \norm{f}_{\px} = 1} \; \err \paren{\Phi, f} , \quad \text{where} \quad \err \paren{ \Phi, f}  = \min_{\vw \in \R^d, \; b \in \R} \; \norm{\vw^{\top} \Phi + b - f}_\px^2 .
\end{equation*}
\end{definition}
The following key result shows that the $\Phi$ that minimizes $\err(\Phi; \fep)$ over all $d$-dimensional encoders must recover the linear space spanned by the $\mu_1,\cdots,\mu_d$. Here $\mu_0$ is excluded since the bias $b$ in the linear predictor implicitly contains $\mu_0$.

\begin{theorem}[Proof in \Cref{app:proof-top-d-optimal}]
\label{thm:top-d-optimal}
Suppose $1 - s_1 \leq \eps \leq 1 - \sqrt{\frac{s_1^2 + s_2^2}{2}}$. 
For any $d$, among all $\Phi = [\phi_1,\cdots,\phi_d]$ where $\phi_i \in \lxp$ , $\Phi$ minimizes $\err(\Phi; \fep)$ if and only if it learns the contexture of $\tp$. The error is given by
\[
\min_{\Phi: \gX \rightarrow \R^d, \; \phi_i \in \lxp} \;  \err \paren{ \Phi; \gF_\epsilon(P^+) } = \frac{s_1^2 - (1 - \eps)^2}{s_1^2 - s_{d+1}^2} .
\]
Conversely, for any $d$-dimensional encoder $\Phi$ and any $\epsilon>0$, there exists $f \in \lxp$ such that $\rho(f,P^+) = 1-\epsilon$, and $\err(\Phi, f) \ge \frac{s_1^2 - (1 - \eps)^2}{s_1^2 - s_{d+1}^2}$.
\end{theorem}
This result has two parts.
First, we show that if $f^*$ is compatible ($f^* \in \fep$), the optimal encoder achieves low error on $f^*$.
Second, we ask what if $f^*$ is incompatible.
We cannot claim that no $\Phi$ works for $f^*$---if one knows $f^*$ a priori, then one can set $\phi_1 = f^*$ to achieve zero error.
Instead, we show that for any $\Phi$, there exists an $f$ with the same compatibility as $f^*$ for which $\Phi$ performs poorly.
Therefore, compatibility reflects whether a context is suitable for a task.

\paragraph{Evaluating an arbitrary encoder.} 
The above result bounds the approximation error of the encoder that learns the contexture.
We can also bound the approximation error of an arbitrary encoder using the contexture theory.
See \Cref{app:arbitrary-encoder}.

\subsection{Implications for Scaling Laws}
\label{sec:scaling-law}
Scaling laws~\cite{kaplan2020scaling} state that the performance of deep neural models such as foundation models grow with their size.
Moreover, models of different architectures learn highly aligned~\cite{pmlr-v97-kornblith19a} representations when scaled up.
\cite{pmlr-v235-huh24a} thus proposed the platonic representation hypothesis that scaling makes representations more aligned with an underlying \textit{reality}, though they did not formally define this reality.

The contexture theory provides a new perspective on the role of scaling.
The function class from which the feature encoders $\Phi$  are trained is a subset of $\lxp$, and as the model gets larger, the function class approaches $\lxp$.
This suggests that scaling brings the learned representation closer to the span of the top-$d$ singular functions of $\tp$, explaining why big models learn aligned representations.
The key difference is that contexts are designed by humans and thus are more subjective than the so-called ``underlying reality''.

We substantiate this extrapolation with an experiment on the \texttt{abalone} dataset from OpenML.
We use KNN ($K=30$) as context, where $\gA = \gX$, and $P^+$ maps $x$ to one of its $K$ nearest neighbors equiprobably. 
We compare two $d$-dimensional representations with $d=128$ learned in the following two ways: (i) Kernel PCA to obtain the exact top-$d$ eigenfunctions of $\tkx$; (ii) non-contrastive learning ($\gL_N$ in \Cref{thm:ssl-singular}) implemented with VICReg \cite{bardes2021vicreg}.
For (ii), we use a fully-connected neural network with Tanh activation, skip connections, and AdamW optimizer \cite{DBLP:journals/corr/KingmaB14,loshchilov2017decoupled}.
We use the same number of training epochs for every model.
For each width and depth, we run the experiments 15 times with different random initializations, and report the average alignment.
See \Cref{app:scaling-law-experiment} for more details.

We measure the correlation between the two representations using two metrics: the classical canonical-correlation analysis (CCA) metric $R_{\rm{CCA}}^2$, and the mutual KNN metric with 10 neighbors as used by \cite{pmlr-v235-huh24a}. 
We center and whiten the representations (making the covariance identity) when using the mutual KNN.
CCA is invariant to all invertible linear transformations $\Phi$, which is ideal because such transformations do not affect the performance of the downstream linear probe, since one can adjust $\mW$ and $\vb$ of the linear probe accordingly.
\cite{pmlr-v97-kornblith19a} also proposed the linear CKA metric, but we do not use it because it is only invariant to orthogonal transformations on $\Phi$.

\begin{figure}[!t]
    \centering
    \begin{tikzpicture}
\begin{axis}[
name=plot1,
 height = .35\linewidth,
    width = .7\linewidth,
    xlabel={Width},
    ylabel={Alignment},
    xlabel style = {
        at={(1.08,-0.05)},
        anchor=south,
    },
        legend style={
        at={(1.1,0.8)},
        anchor=north west,
        cells={anchor=west},
    },
title style={at={(0.5,-0.3)}, anchor=north},
    xmin=0, xmax=12,
    ymin=0.6, ymax=0.92,
    xtick={0,1,2,3,4,5,6,7,8,9,10,11,12},
    xticklabels={64,128,256,512,1024,2048,4000,8000,12000,16000,20000,25000,30000},
    x tick label style={rotate=45,font=\small},
    ytick distance = 0.1,
]
    \addplot[very thick,blue] table [
        x expr=\coordindex, 
        y=CCA,            
        col sep=comma       
    ] {data/topd.csv};
    \addplot[very thick,teal] table [
        x expr=\coordindex, 
        y=CCA,            
        col sep=comma       
    ] {data/topd2.csv};
    \addplot[very thick,red] table [
        x expr=\coordindex, 
        y=CCA,            
        col sep=comma       
    ] {data/topd3.csv};
    \addplot[very thick,blue,dashed] table [
        x expr=\coordindex, 
        y=MKNN,            
        col sep=comma       
    ] {data/topd.csv};
    \addplot[very thick,teal,dashed] table [
        x expr=\coordindex, 
        y=MKNN,            
        col sep=comma       
    ] {data/topd2.csv};
    \addplot[very thick,red,dashed] table [
        x expr=\coordindex, 
        y=MKNN,            
        col sep=comma       
    ] {data/topd3.csv};
    \legend{1,2,3};
\end{axis}
\node[] at (11.55,3.6) {Depth};

\begin{axis}[
name=plot2,
at={($(plot1.south west)+(0,-170)$)},
 height = .35\linewidth,
    width = .7\linewidth,
    xlabel={Depth},
    ylabel={Alignment},
    xlabel style = {
        at={(1.08,-0.05)},
        anchor=south,
    },
    x tick label style={rotate=45,font=\small},
title style={at={(0.5,-0.3)}, anchor=north},
    xmin=0, xmax=9,
    ymin=0.55, ymax=0.9,
    xtick={0,1,2,3,4,5,6,7,8,9},
    xticklabels={1,2,4,8,16,32,64,100,150,200},
    ytick distance = 0.1,
        legend style={
        at={(1.1,0.8)},
        anchor=north west,
        cells={anchor=west},
    },
]
    \addplot[very thick,blue] table [
        x expr=\coordindex,  
        y=CCA,            
        col sep=comma       
    ] {data/topd-512.csv};
    \addplot[very thick,teal] table [
        x expr=\coordindex,  
        y=CCA,            
        col sep=comma       
    ] {data/topd-1024.csv};
    \addplot[very thick,red] table [
        x expr=\coordindex,  
        y=CCA,            
        col sep=comma       
    ] {data/topd-2048.csv};
    \legend{512,1024,2048};
    \addplot[very thick,blue,dashed] table [
        x expr=\coordindex,  
        y=MKNN,            
        col sep=comma       
    ] {data/topd-512.csv};
    \addplot[very thick,teal,dashed] table [
        x expr=\coordindex,  
        y=MKNN,            
        col sep=comma       
    ] {data/topd-1024.csv};
    \addplot[very thick,red,dashed] table [
        x expr=\coordindex,  
        y=MKNN,            
        col sep=comma       
    ] {data/topd-2048.csv};
\end{axis}
\node[] at (11.8,-2.2) {Width};
\end{tikzpicture}
    \vskip -.15in
\caption{Alignment between the learned representation and the top-$d$ eigenfunctions of $\tkx$ on the \texttt{abalone} dataset. Solid curves: CCA. Dashed curves: mutual KNN. Depth here means the number of hidden layers.}
\label{fig:scaling-law}
\end{figure}
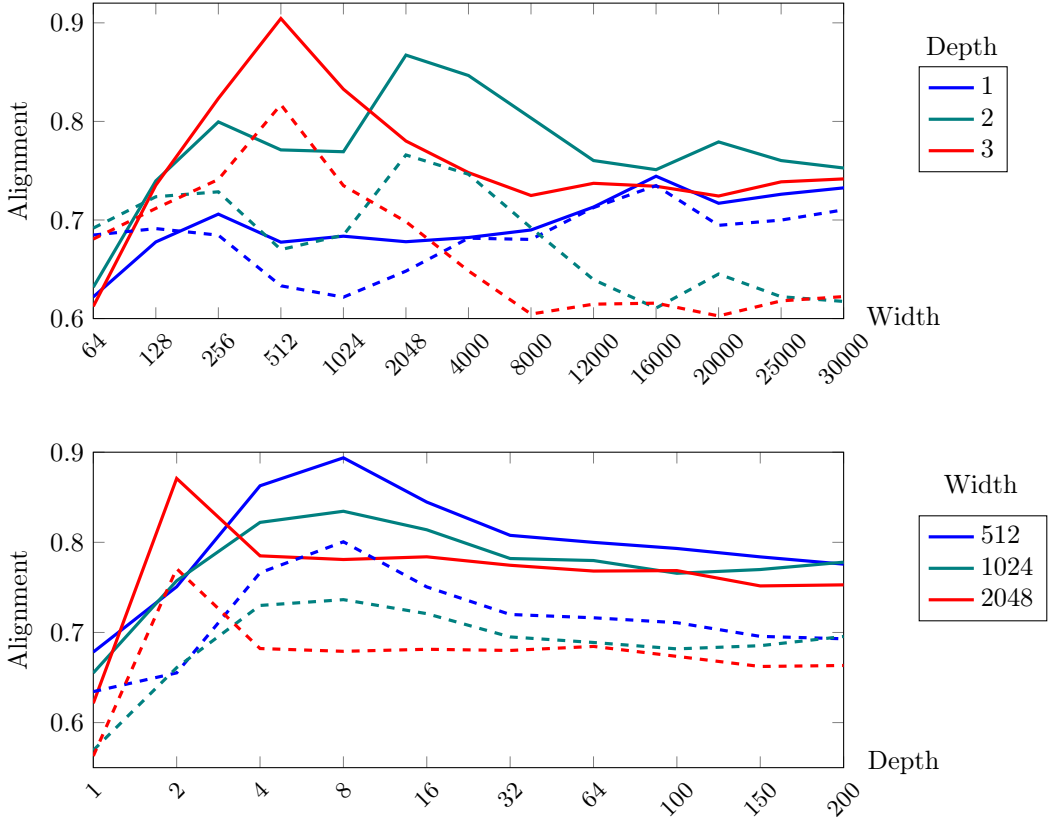

\Cref{fig:scaling-law} plots the alignment between the exact top-$d$ eigenfunctions and the learned deep representation while varying the depth and width of the neural network.
We can see that when depth and width are chosen correctly, the CCA can be as high as 0.9, and the mutual KNN can be higher than 0.8.
Note that these alignment metric values are very high.
For example, in \cite{pmlr-v235-huh24a}, the mutual KNN metric value is usually below 0.2.
Hence, the representation learned by the neural network is highly aligned with the top-$d$ eigenfunctions.

The top plot studies neural networks with increasing widths.
We observe that when the neural network is relatively narrow, increasing the width improves alignment.
However, once the neural network is sufficiently wide, further increasing the width may have a negative effect.
For example, when the depth is 3, the alignment is the highest when the width is 512, and the alignment becomes lower when the network is wider than 512.
Since increasing the width can only make the function class of $\Phi$ larger, this phenomenon is not due to the expressivity of the neural network.
We hypothesize that it arises from optimization difficulty, that is larger models are harder to train effectively.
Consequently, with the same number of pretraining steps, a larger model will be farther away from the minima, leading to a reduced alignment.

The bottom plot studies neural networks with increasing depths, and a similar trend is observed.
When the network is shallow, increasing the depth improves the alignment.
However, once the network is sufficiently deep, further increasing the depth may have a negative effect.
We also observe from the bottom plot that a width-512 network has higher alignment than widths 1024 and 2048.
In addition, the alignment does not reach 1.
This is expected as the model is non-convex, so the true optima (the precise top-$d$ eigenspace) cannot be found by gradient-based methods.

In summary, we draw two conclusions from this experiment:
(i) the representation learned by a large neural network is highly aligned with the top-$d$ eigenfunctions;
(ii) once the neural network is wide and deep enough, further increasing its size does not improve the alignment higher, and may even degrade it.
Hence, we put forward the following argument about the scaling law:
Once the model is large enough such that $\Phi$ is already highly aligned with the top-$d$ eigenfunctions, further increasing the model size inevitably yields diminishing returns.

\section{Context Evaluation}
\label{sec:context-evaluation}

Creating \emph{useful} contexts that produce better representations is a challenging open problem. 
In this section, we take a first step by studying \emph{when} a context is useful and \emph{how} to efficiently evaluate its usefulness. 
The key result is that the usefulness of a context is largely determined by the association level between $X$ and $A$, and \textbf{a useful context should have a moderate association}.
The association level affects the shape of the spectrum, that is the decay rate of the singular values.
We propose a metric that only uses the singular values to evaluate the usefulness of a context, without knowledge of the downstream task.
Then, we empirically verify that this metric has a strong correlation with the actual performance of the encoder on many real datasets.
As such, the proposed metric can help practitioners to select among various pretraining methods or hyperparameter settings efficiently.

\subsection{The Effect of Context Association}
\label{sec:association}

A useful context should provide sufficient training signals that are easy for the model to capture.
If the association between $X$ and $A$ of a context is too weak, then the signals will be insufficient.
If the association is too strong, then capturing the signals will be too hard.
The association level affects the spectrum of the context---the stronger the association, the slower the decay of the singular values. As an illustration, \Cref{fig:taud} (top) displays the spectra of contexts with weak, moderate, and strong association, from left to right.

\paragraph{Case 1: Weak association.} 
Consider the extreme case where $A$ is independent of $X$.
This context is clearly useless because it provides no information.
In this case, only the trivial singular function $\mu_0 \equiv 1$ has a positive singular value; all the other singular values are $0$.
When $X$ and $A$ are nearly independent, $\kx(x,x')$ is very close to $1$, which causes the singular values to decay too fast.
Formally, we have:
\begin{lemma} [Proof in \Cref{app:proof-lemma-weak_association}] \label{lemma:weak_association}
When $|\kx(x,x')-1| < \epsilon$ for all $x,x' \in \gX$, we have $ \sum_{i >0} s_i^2  < \epsilon$.
\end{lemma}
In \Cref{app:context_evaluation}, we empirically verify that low association leads to a small $|\kx(x,x')-1|$ for all $x,x'$.
In such a scenario, $\fep$ is a very small set, so very few tasks are compatible with and can benefit from the context.

\paragraph{Case 2: Strong association.} 
Contexts with very strong associations, whose singular values decay too slowly, are not useful.
For example, the extreme case $A = X$ is clearly useless.
There are two reasons: (i) for upstream, slow decay implies that there will be very non-smooth singular functions with large singular values, which are difficult to learn; (ii) for downstream, a larger $d$ is needed, as more singular functions have non-trivial contributions to the kernel, and it leads to a higher sample complexity.
In \Cref{app:context_evaluation}, we empirically verify that kernel $\kx$ has a high Lipschitz constant when the association is strong, meaning that the kernel is non-smooth and thus the singular functions are non-smooth.

\begin{figure}[!t]
    \centering
    \includegraphics[width=\textwidth]{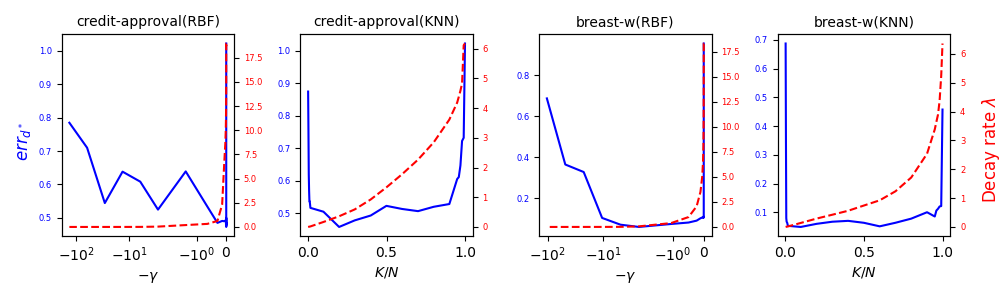}
    \caption{Association level vs, prediction error $\err_{d^*} = \min_d \err_d$ when $\Phi$ is d-dimensional, and the decay rate $\lambda$ for the RBF and KNN contexts on \texttt{credit-approval} and \texttt{breast-w} datasets. A larger $-\gamma$ (or $K/N$) indicates a lower association level, while a smaller $-\gamma$ (or $K/N$) corresponds to a higher association level. Across all four figures, we observe a U-shaped trend: prediction error increases at both extremes of low and high association. The estimated decay rate $\lambda$ consistently captures this behavior, serving as a proxy for the level of association.  }
    \label{fig:association_vs_decay_rate_err_main}
\end{figure}

\paragraph{Quantitative measurements for level of association.}
\label{main:measurement_for_association}
While mutual information captures mutual dependence between random variables, estimating it from samples remains a longstanding challenge as it requires the joint density function to be known~\cite{paninski2003estimation}. As an alternative approach, we propose to use the decay rate of singular values $(s_i)_{i\geq 0}$ as an indicator of the strength of association.

To estimate the decay rate $\lambda$, we assume that the singular values decay exponentially and fit the regression model $s_i^2 = \exp{(-\lambda i)}$. When $\lambda$ is large, it indicates a fast decay rate and the context has a low association. Conversely, when $\lambda$ is small, it implies a slow decay rate and highly associated context. In \Cref{fig:association_vs_decay_rate_err_main}, we empirically demonstrate that the estimated decay rate closely reflects the level of association, and that both extremely strong and weak associations can lead to degraded empirical performance. In \Cref{app:context_evaluation}, we observe similar trends for all $28$ datasets and propose alternative metrics for detecting the level of associations.

\subsection{Task-agnostic Evaluation of Contexts}
\label{sec:metric}

A good measurement of context usefulness should be task-agnostic, because we would like the pretrained encoder to be transferable to a variety of tasks, which we might not know at pretrain time.
Note that for any task-agnostic metric, one can adversarially create a task for which the metric fails, so there is no universal task-agnostic metric.
However, a metric can still be very useful if it provides guidance for most real tasks. To this end, we set up the following desiderata.
\begin{enumerate}[itemsep=0pt,leftmargin=*]
    \item \emph{Approximation-estimation trade-off.} The prediction error of a linear model fit on the encoder can be decomposed as the sum of approximation error and estimation error. The former decreases with $d$, while the latter increases with $d$. The metric should reflect this trade-off.
    \item \emph{Efficient computation.} Note that estimating the singular functions is as hard as training an encoder. Therefore, we propose to use a metric that only uses the singular values of $\tp$, which can be more efficiently estimated.
\end{enumerate}
With these desiderata in mind, we define our metric as
\begin{equation}
\label{eqn:metric}
    \tau_d = \frac{1}{1 - s_{d+1}^2} \ + \ \beta \ \frac{\sum_{i=1}^{d} s_i^2}{\sum_{i=1}^{d_0} s_i^2} \ , \qquad  \tau = \min_d \tau_d  ,
\end{equation}
where $\beta > 0$ is a parameter, and $d_0$ is the maximum embedding dimension we consider.
Typically $d_0$ ranges from $512$ to $8192$.
We choose $\beta = 1$ and $d_0 = 512$ in our experiments.
$\tau_d$ is a proxy of the prediction error when the embedding dimension is $d$.
Thus, the $d$ that minimizes $\tau_d$ can be viewed as the optimal embedding dimension predicted by the metric, and $\tau$ evaluates the context when $d$ is chosen optimally.

\paragraph{Derivation of the metric.} 
Let the target function be $f^* = f_0 + f_1$, where $\dotp{f_0,f_1}_{\px} = 0$, $f_1$ is compatible with the context, and $f_0$ is not compatible with the context. The prediction error can then be decomposed into three components:
\begin{enumerate}[label=(\roman*), itemsep=0pt]
    \item The approximation error of $f_1$
    \item The approximation error of $f_0$
    \item The estimation error
\end{enumerate}
By \Cref{thm:top-d-optimal}, component (i) can be bounded by $\frac{s_1^2 - (1-\epsilon)^2}{s_1^2 - s_{d+1}^2}$. In practice, $s_1$ is usually very close to $1$, and we simplify this bound to the first term of \Cref{eqn:metric}, up to a constant factor.
For component (ii), stronger associations imply that more tasks are compatible with the context, reducing this approximation error. Thus, this component should be negatively correlated with $\sum_{i=1}^{d_0} s_i^2$.
Component (iii), the estimation error, increases with stronger associations, since higher association typically requires a larger $d$, and thus greater sample complexity. 
Based on the results in \cite{zhai2023understanding}, this component can be essentially understood as positively correlated with $\sum_{i=1}^{d} s_i^2$.
The second term in \Cref{eqn:metric} combines the contributions from components (ii) and (iii), and is designed to be bounded by $1$.
This metric can be efficiently estimated. It only requires the top-$d_0$ eigenfunctions of $\tkx$, which can be estimated in $O(m^3)$ time using a random subset of $m = \Theta(d_0 \log d_0)$ samples. See \Cref{app:efficient-estimation} for details.

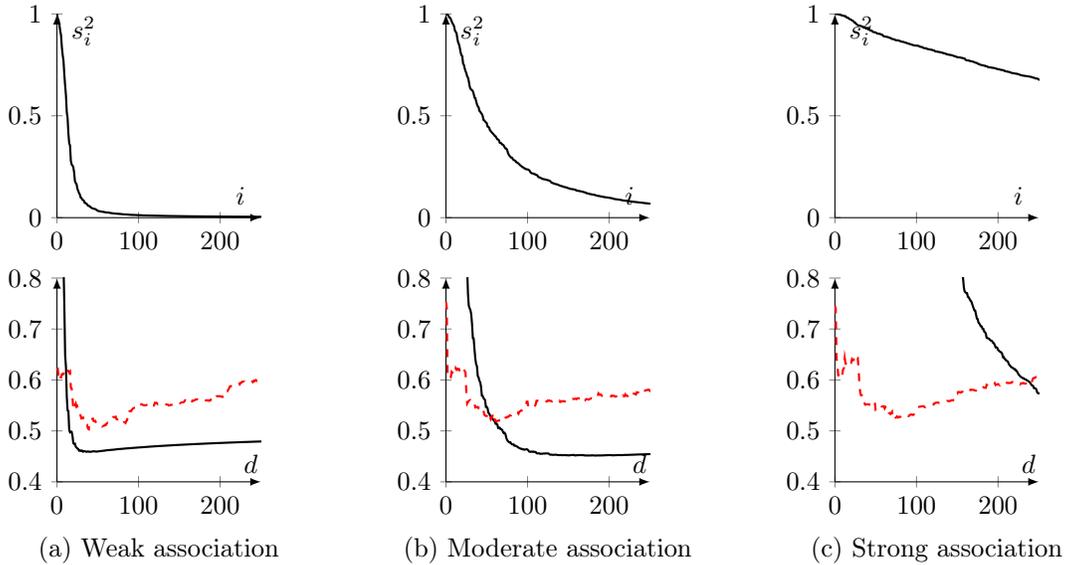
\begin{figure}[!t]
    \centering
\begin{tikzpicture}
\begin{axis}[
name=plot1,
 height = .26\linewidth,
    width = .26\linewidth,
    axis lines=left,  
    axis line style={-latex},  
    xlabel={$i$},
    ylabel={$s_i^2$},
    ylabel style={
        at={(0.13,0.8)}, 
        anchor=south,  
        rotate=270  
    },
    xlabel style = {
        at={(0.9,0.02)},
        anchor=south,
    },
title style={at={(0.5,-0.3)}, anchor=north},
    xmin=0, xmax=250,
    ymin=0, ymax=1,
    xtick distance = 100,
    ytick distance = 0.5,
]
    \addplot[thick] table [
        x expr=\coordindex,  
        y=value,            
        col sep=comma       
    ] {data/ag1.csv};
\end{axis}

\begin{axis}[
name=plot2,
at={($(plot1.south east)+(70,0)$)},
 height = .26\linewidth,
    width = .26\linewidth,
    axis lines=left,  
    axis line style={-latex},  
    xlabel={$i$},
    ylabel={$s_i^2$},
    ylabel style={
        at={(0.13,0.8)}, 
        anchor=south,  
        rotate=270  
    },
    xlabel style = {
        at={(0.9,0.02)},
        anchor=south,
    },
title style={at={(0.5,-0.3)}, anchor=north},
    xmin=0, xmax=250,
    ymin=0, ymax=1,
    xtick distance = 100,
    ytick distance = 0.5,
]
    \addplot[thick] table [
        x expr=\coordindex,  
        y=value,            
        col sep=comma       
    ] {data/ag3.csv};
\end{axis}

\begin{axis}[
name=plot3,
at={($(plot2.south east)+(70,0)$)},
 height = .26\linewidth,
    width = .26\linewidth,
    axis lines=left,  
    axis line style={-latex},  
    xlabel={$i$},
    ylabel={$s_i^2$},
    ylabel style={
        at={(0.13,0.8)}, 
        anchor=south,  
        rotate=270  
    },
    xlabel style = {
        at={(0.9,0.02)},
        anchor=south,
    },
title style={at={(0.5,-0.3)}, anchor=north},
    xmin=0, xmax=250,
    ymin=0, ymax=1,
    xtick distance = 100,
    ytick distance = 0.5,
]
    \addplot[thick] table [
        x expr=\coordindex,  
        y=value,            
        col sep=comma       
    ] {data/ag5.csv};
\end{axis}

\begin{axis}[
name=plot4,
at={($(plot1.south west)+(0,-100)$)},
 height = .26\linewidth,
    width = .26\linewidth,
    axis lines=left,  
    axis line style={-latex},  
    xlabel={$d$},
    xlabel style = {
        at={(0.95,0)},
        anchor=south,
    },
title style={at={(0.5,-0.3)}, anchor=north},
title={(a) Weak association},
    xmin=0, xmax=250,
    ymin=0.4, ymax=0.8,
    xtick distance = 100,
    ytick distance = 0.1,
]
    \addplot[thick] table [
        x expr=\coordindex,  
        y expr={min(\thisrow{tau} / 4, 1.4)},            
        col sep=comma       
    ] {data/ag2.csv};
    \addplot[thick, dashed, red] table [
        x expr=\coordindex,  
        y=mse,            
        col sep=comma       
    ] {data/ag2.csv};
\end{axis}

\begin{axis}[
name=plot5,
at={($(plot4.south east)+(70,0)$)},
 height = .26\linewidth,
    width = .26\linewidth,
    axis lines=left,  
    axis line style={-latex},  
    xlabel={$d$},
    xlabel style = {
        at={(0.95,0)},
        anchor=south,
    },
title style={at={(0.5,-0.3)}, anchor=north},
title={(b) Moderate association},
    xmin=0, xmax=250,
    ymin=0.4, ymax=0.8,
    xtick distance = 100,
    ytick distance = 0.1,
]
    \addplot[thick] table [
        x expr=\coordindex,  
        y expr={min(\thisrow{tau} / 4, 1.4)},     
        col sep=comma       
    ] {data/ag4.csv};
    \addplot[thick, dashed, red] table [
        x expr=\coordindex,  
        y=mse,    
        col sep=comma       
    ] {data/ag4.csv};
\end{axis}

\begin{axis}[
name=plot6,
at={($(plot5.south east)+(70,0)$)},
 height = .26\linewidth,
    width = .26\linewidth,
    axis lines=left,  
    axis line style={-latex},  
    xlabel={$d$},
    xlabel style = {
        at={(0.95,0)},
        anchor=south,
    },
title style={at={(0.5,-0.3)}, anchor=north},
title={(c) Strong association},
    xmin=0, xmax=250,
    ymin=0.4, ymax=0.8,
    xtick distance = 100,
    ytick distance = 0.1,
]
    \addplot[thick] table [
        x expr=\coordindex,  
        y expr={min(\thisrow{tau} / 6, 1.4)},            
        col sep=comma       
    ] {data/ag6.csv};
    \addplot[thick, dashed, red] table [
        x expr=\coordindex,  
        y=mse,            
        col sep=comma       
    ] {data/ag6.csv};
\end{axis}
\end{tikzpicture}
    \caption{Metric illustration on \texttt{abalone}. \textbf{Top row:} context spectra. \textbf{Bottom row:} black solid curves are $\tau_d$ divided by $6$; red dashed curves are the actual downstream prediction error. We divide $\tau_d$ by $6$ to fit it in the same plot.}
    \label{fig:taud}
\end{figure}

Next, we empirically examine $\tau_d$ on two datasets.
First, we apply the metric to the \texttt{abalone} dataset and use KNN as the context, similar to \Cref{sec:scaling-law}.
We adjust the association of the context by changing $K$.
In particular, we choose $K=150$ (weak), $K = 30$ (moderate) and $K = 5$ (strong).
We obtain the exact eigenvalues and eigenfunctions of $\tkx$ using kernel PCA.
In \Cref{fig:taud}, we plot the spectra of the three contexts in the top row.
Then, in the bottom row, we compare $\tau_d$ against the prediction error of the linear probe under different $d$.
We can see that when the association is weak or moderate, $\tau_d$ first decreases and then increases, which tracks the actual error.
However, when the association is too strong, $\tau_d$ monotonically decreases with $d$, and it cannot track the actual error.

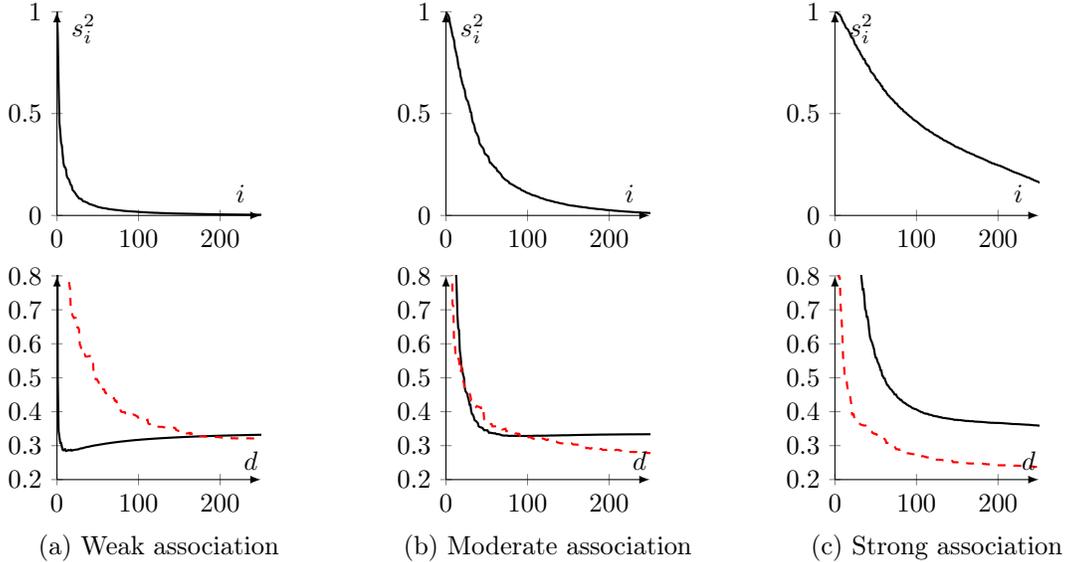
\begin{figure}[!t]
    \centering
\begin{tikzpicture}
\begin{axis}[
name=plot1,
 height = .26\linewidth,
    width = .26\linewidth,
    axis lines=left,  
    axis line style={-latex},  
    xlabel={$i$},
    ylabel={$s_i^2$},
    ylabel style={
        at={(0.13,0.8)}, 
        anchor=south,  
        rotate=270  
    },
    xlabel style = {
        at={(0.9,0.02)},
        anchor=south,
    },
title style={at={(0.5,-0.3)}, anchor=north},
    xmin=0, xmax=250,
    ymin=0, ymax=1,
    xtick distance = 100,
    ytick distance = 0.5,
]
    \addplot[thick] table [
        x expr=\coordindex,  
        y=eig,            
        col sep=comma       
    ] {data/mnist-05.csv};
\end{axis}

\begin{axis}[
name=plot2,
at={($(plot1.south east)+(70,0)$)},
 height = .26\linewidth,
    width = .26\linewidth,
    axis lines=left,  
    axis line style={-latex},  
    xlabel={$i$},
    ylabel={$s_i^2$},
    ylabel style={
        at={(0.13,0.8)}, 
        anchor=south,  
        rotate=270  
    },
    xlabel style = {
        at={(0.9,0.02)},
        anchor=south,
    },
title style={at={(0.5,-0.3)}, anchor=north},
    xmin=0, xmax=250,
    ymin=0, ymax=1,
    xtick distance = 100,
    ytick distance = 0.5,
]
    \addplot[thick] table [
        x expr=\coordindex,  
        y=eig,            
        col sep=comma       
    ] {data/mnist-02.csv};
\end{axis}

\begin{axis}[
name=plot3,
at={($(plot2.south east)+(70,0)$)},
 height = .26\linewidth,
    width = .26\linewidth,
    axis lines=left,  
    axis line style={-latex},  
    xlabel={$i$},
    ylabel={$s_i^2$},
    ylabel style={
        at={(0.13,0.8)}, 
        anchor=south,  
        rotate=270  
    },
    xlabel style = {
        at={(0.9,0.02)},
        anchor=south,
    },
title style={at={(0.5,-0.3)}, anchor=north},
    xmin=0, xmax=250,
    ymin=0, ymax=1,
    xtick distance = 100,
    ytick distance = 0.5,
]
    \addplot[thick] table [
        x expr=\coordindex,  
        y=eig,            
        col sep=comma       
    ] {data/mnist-005.csv};
\end{axis}

\begin{axis}[
name=plot4,
at={($(plot1.south west)+(0,-100)$)},
 height = .26\linewidth,
    width = .26\linewidth,
    axis lines=left,  
    axis line style={-latex},  
    xlabel={$d$},
    xlabel style = {
        at={(0.95,0)},
        anchor=south,
    },
title style={at={(0.5,-0.3)}, anchor=north},
title={(a) Weak association},
    xmin=0, xmax=250,
    ymin=0.2, ymax=0.8,
    xtick distance = 100,
    ytick distance = 0.1,
]
    \addplot[thick] table [
        x expr=\coordindex,  
        y expr={min(\thisrow{taud} / 6, 1.4)},            
        col sep=comma       
    ] {data/mnist-05.csv};
    \addplot[thick, dashed, red] table [
        x expr=\coordindex,  
        y=mse,            
        col sep=comma       
    ] {data/mnist-05.csv};
\end{axis}

\begin{axis}[
name=plot5,
at={($(plot4.south east)+(70,0)$)},
 height = .26\linewidth,
    width = .26\linewidth,
    axis lines=left,  
    axis line style={-latex},  
    xlabel={$d$},
    xlabel style = {
        at={(0.95,0)},
        anchor=south,
    },
title style={at={(0.5,-0.3)}, anchor=north},
title={(b) Moderate association},
    xmin=0, xmax=250,
    ymin=0.2, ymax=0.8,
    xtick distance = 100,
    ytick distance = 0.1,
]
    \addplot[thick] table [
        x expr=\coordindex,  
        y expr={min(\thisrow{taud} / 6, 1.4)},     
        col sep=comma       
    ] {data/mnist-02.csv};
    \addplot[thick, dashed, red] table [
        x expr=\coordindex,  
        y=mse,    
        col sep=comma       
    ] {data/mnist-02.csv};
\end{axis}

\begin{axis}[
name=plot6,
at={($(plot5.south east)+(70,0)$)},
 height = .26\linewidth,
    width = .26\linewidth,
    axis lines=left,  
    axis line style={-latex},  
    xlabel={$d$},
    xlabel style = {
        at={(0.95,0)},
        anchor=south,
    },
title style={at={(0.5,-0.3)}, anchor=north},
title={(c) Strong association},
    xmin=0, xmax=250,
    ymin=0.2, ymax=0.8,
    xtick distance = 100,
    ytick distance = 0.1,
]
    \addplot[thick] table [
        x expr=\coordindex,  
        y expr={min(\thisrow{taud} / 6, 1.4)},            
        col sep=comma       
    ] {data/mnist-005.csv};
    \addplot[thick, dashed, red] table [
        x expr=\coordindex,  
        y=mse,            
        col sep=comma       
    ] {data/mnist-005.csv};
\end{axis}
\end{tikzpicture}
    \caption{Metric illustration on \texttt{MNIST}, similar to \Cref{fig:taud}.}
    \label{fig:taud-mnist}
\end{figure}

Second, we apply the metric to the \texttt{MNIST} dataset.
The context is random cropping with crop ratio $\alpha$.
We adjust the association of the context by changing $\alpha$.
In particular, we choose $\alpha=0.5$ (weak), $\alpha=0.2$ (moderate) and $\alpha=0.05$ (strong).
Since kernel PCA is not scalable to datasets as large as MNIST, we instead train a neural network.
Specifically, we train a LeNet \cite{lecun1998gradient} using the non-contrastive learning objective ($\gL_N$ in \Cref{thm:ssl-singular}) and the AdamW optimizer.
Then, we estimate the top eigenvalues using the method in \Cref{app:efficient-estimation}.
The downstream task is a binary classification task---whether the digit is greater than 4.
After pretraining, a linear probe is fit on top of $\Phi$ using ridge regression.
The result is plotted in \Cref{fig:taud-mnist}.

From \Cref{fig:taud-mnist}, we can see that when the association is not too strong, $\tau_d$ first decreases and then increases, similar to \Cref{fig:taud}.
However, on \texttt{MNIST}, the downstream error monotonically decreases with $d$, unlike \texttt{abalone}.
This disparity is due to the difference between the two downstream tasks.
To demonstrate this, in \Cref{fig:mnist-compare} we plot the cosine similarity between the target function $f^*$ and the estimated $i$-th eigenfunction on the two datasets.
We can see that the variance of $f^*$ on \texttt{abalone} is mostly concentrated on the top-$5$ eigenfunctions, with the first cosine similarity being almost $0.5$.
In contrast, the variance of $f^*$ on \texttt{MNIST} is more scattered, and the cosine similarity is still close to $0.1$ for the $150$-th eigenfunction.
Consequently, having a large $d$ on \texttt{abalone} will have a little impact on the approximation error but will increase the estimation error significantly.
On the other hand, having a larger $d$ on \texttt{MNIST} will decrease the approximation error more than it increases the estimation error, which is why the total error monotonically decreases with $d$.

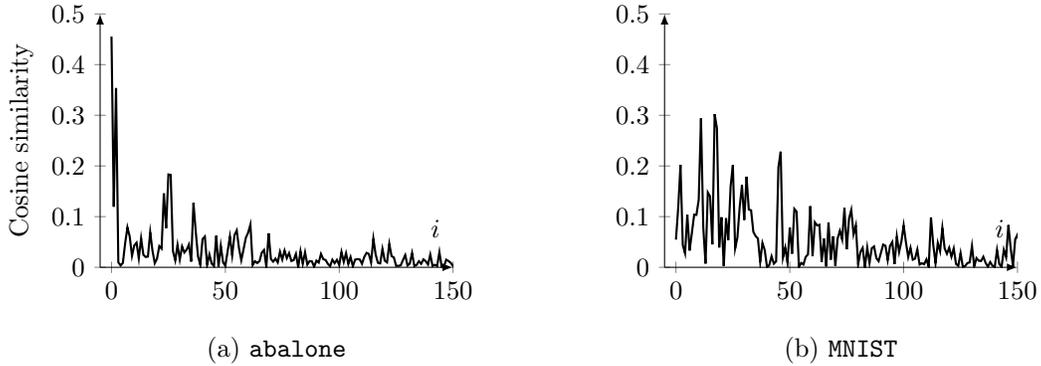
\begin{figure}[!t]
    \centering
\begin{tikzpicture}
\begin{axis}[
name=plot1,
title={(a) \texttt{abalone}},
 height = .3\linewidth,
    width = .38\linewidth,
    axis lines=left,  
    axis line style={-latex},  
    xlabel={$i$},
    ylabel={Cosine similarity},
    xlabel style = {
        at={(0.95,0.08)},
        anchor=south,
    },
title style={at={(0.5,-0.3)}, anchor=north},
    xmin=-5, xmax=150,
    ymin=0, ymax=0.5,
    xtick distance = 50,
    ytick distance = 0.1,
]
    \addplot[thick] table [
        x expr=\coordindex,  
        y=ui,            
        col sep=comma       
    ] {data/abalone-ui.csv};
\end{axis}

\begin{axis}[
name=plot2,
title={(b) \texttt{MNIST}},
at={($(plot1.south east)+(80,0)$)},
 height = .3\linewidth,
    width = .38\linewidth,
    axis lines=left,  
    axis line style={-latex},  
    xlabel={$i$},
    xlabel style = {
        at={(0.95,0.08)},
        anchor=south,
    },
title style={at={(0.5,-0.3)}, anchor=north},
    xmin=-5, xmax=150,
    ymin=0, ymax=0.5,
    xtick distance = 50,
    ytick distance = 0.1,
]
    \addplot[thick] table [
        x expr=\coordindex,  
        y=ui,            
        col sep=comma       
    ] {data/mnist-05.csv};
\end{axis}

\end{tikzpicture}
    \caption{Comparison of the downstream task between \texttt{abalone} and \texttt{MNIST}. The $y$-axis is the cosine similarity between the downstream task and the $i$-th eigenfunction.}
    \label{fig:mnist-compare}
\end{figure}

The takeaway from this experiment is that, while a context with moderate association is generally good, its effectiveness ultimately depends on the specific downstream task.

For example, on \texttt{abalone} the weakest context actually leads to the lowest error, because the variance of $f^*$ is concentrated on the top-$5$ eigenfunctions.
On the other hand, on \texttt{MNIST} the strongest context leads to the lowest error, because the variance of $f^*$ is scattered among a lot of features, and a stronger association allows more features to be discovered.
Hence, no evaluation metric would universally work for all contexts and downstream tasks, but a metric would still be useful if it correlates well with the actual error in most scenarios, and thus can provide insights into choosing the right context and the right hyperparameters, such as the mask or crop ratio.

\subsection{Empirical Verification}
\label{sec:empirical_verification}
In this section, we examine whether our metric correlates with the encoder's performance on real datasets.
In practice, the prediction error is influenced by many factors.
To create a setting where all factors but the context are controlled, we let the encoder be the exact top-$d$ singular functions obtained by kernel PCA.

\begin{table}[!t]
\caption{Correlation between $\tau$ and the actual error $\err_{d^*}$ on all 4 types of contexts.}
\label{tab:correlation}
\begin{center}
\begin{tabular}{llll|ll}
\toprule
Dataset & Size ($\uparrow$) & \#Feature & Type & Pearson  & Distribution  \\
\midrule
credit-approval    & 690   & 15   & Cls & 0.583  & 0.683  \\ 
breast-w           & 699   & 9    & Cls & 0.072  & 0.255  \\ 
diabetes           & 768   & 8    & Cls & 0.737  & 0.740  \\ 
solar\_flare       & 1066  & 10   & Reg & 0.019  & 0.262  \\ 
Moneyball          & 1232  & 14   & Reg & 0.680  & 0.650  \\ 
yeast              & 1269  & 8    & Cls & 0.221  & 0.256  \\ 
cmc                & 1473  & 9    & Cls & 0.867  & 0.860  \\ 
Wine               & 1599  & 11   & Reg & -0.084 & 0.212  \\ 
scene              & 2407  & 299  & Cls & 0.608  & 0.685  \\ 
dna                & 3186  & 180  & Cls & 0.881  & 0.843  \\ 
splice             & 3190  & 60   & Cls & 0.831  & 0.801  \\ 
kr-vs-kp           & 3196  & 36   & Cls & 0.543  & 0.512  \\ 
abalone            & 4177  & 8    & Reg & 0.028  & 0.470  \\  
spambase           & 4601  & 57   & Cls & 0.775  & 0.858  \\  
colleges           & 7603  & 44   & Reg & 0.155  & 0.387  \\  
mushroom           & 8124  & 22   & Cls & 0.185  & 0.340  \\  
kin8nm             & 8192  & 8    & Reg & 0.805  & 0.760  \\ 
pumadyn32nh        & 8192  & 32   & Reg & 0.938  & 0.961  \\ 
cpu\_activity      & 8192  & 21   & Reg & 0.709  & 0.825  \\ 
SpeedDating        & 8378  & 120  & Cls & 0.590  & 0.656  \\ 
grid\_stability    & 10000 & 12   & Reg & 0.925  & 0.911  \\ 
sulfur             & 10081 & 6    & Reg & -0.180 & 0.487  \\ 
brazilian\_houses  & 10692 & 9    & Reg & -0.290 & 0.563  \\ 
fifa               & 19178 & 28   & Reg & -0.349 & 0.663  \\ 
superconductivity  & 21263 & 81   & Reg & 0.141  & 0.367  \\ 
kings\_county      & 21613 & 21   & Reg & 0.842  & 0.882  \\ 
health\_insurance  & 22272 & 11   & Reg & 0.601  & 0.749  \\  
cps88wages         & 28155 & 6    & Reg & 0.250  & 0.479  \\ 
\midrule
\multicolumn{4}{r|}{\textbf{Mean}} & 0.431 &  0.611 \\ 
\multicolumn{4}{r|}{\textbf{Median}} & 0.587 & 0.659 \\ 
\bottomrule
\end{tabular}
\end{center}
\end{table}

Each dataset is randomly split into a pretrain set, a downstream labeled set, and a test set.
The downstream linear predictor is fit via ridge regression.
Hyperparameter grid search is conducted at both encoder learning and downstream stages.
The evaluation metric is the mean squared error. 
Let $\err_d$ be the actual prediction error when $\Phi$ is $d$-dimensional.
We test $d$ up to $d_0 = 512$.
Let $d^*$ be the one that minimizes $\err_d$.
We use the following four types of contexts. 
\begin{itemize}[itemsep=0pt]
    \item RBF kernels: $k(x,a) = \exp(-\gamma \norm{x-a}^2)$. We define $P^+$ as $P^+\sparen{a}{x} \propto k(x,a)$ for each $x$.
    \item KNN: $P^+\sparen{a}{x} = K^{-1}$ if $a$ is a KNN of $x$, else 0.
    \item RBF$_{\textrm{mask}}$: First, randomly mask 20\% of the features, and then apply RBF kernels to the other features. Specifically, we randomly draw 50 masks, and use the average of $P^+$ over all masks as the context. We do not use masking alone because its association is too strong.
    \item KNN$_{\textrm{mask}}$: 20\% random masking and then apply KNN.
\end{itemize}
For each of these contexts, $\gA = \gX$.
For each type, we use 35 contexts by adjusting $\gamma$ for RBF kernels and $K$ for KNN. By doing so, we adjust the association level between $X$ and $A$.
We make sure that contexts in every type range from very weak to very strong association.

In \Cref{tab:correlation} we report the correlation between $\tau$ and $\err_{d^*}$ over all 140 contexts from the four types on 28 classification (Cls) and regression (Reg) datasets from OpenML \cite{OpenML2013} that are widely used in machine learning research.
The most common metric is the Pearson correlation, but it can only detect linear correlations, while the correlation between $\tau$ and $\err_{d^*}$ is not necessarily linear. Thus, we also report the distance correlation \cite{648ae489-f5d9-3200-b1e9-ab0b264416e5}, another common metric that can detect non-linear correlations, but it cannot tell if the correlation is positive or negative because it is always non-negative.

The median reported in the table shows that on more than half of the datasets, the Pearson correlation is over $0.5$, 
which is generally considered a strong correlation.
The distance correlation is even higher.
As expected, the metric does not work on all datasets.
For example, the Pearson correlation is very negative on \texttt{brazilian\_houses} and \texttt{fifa}.

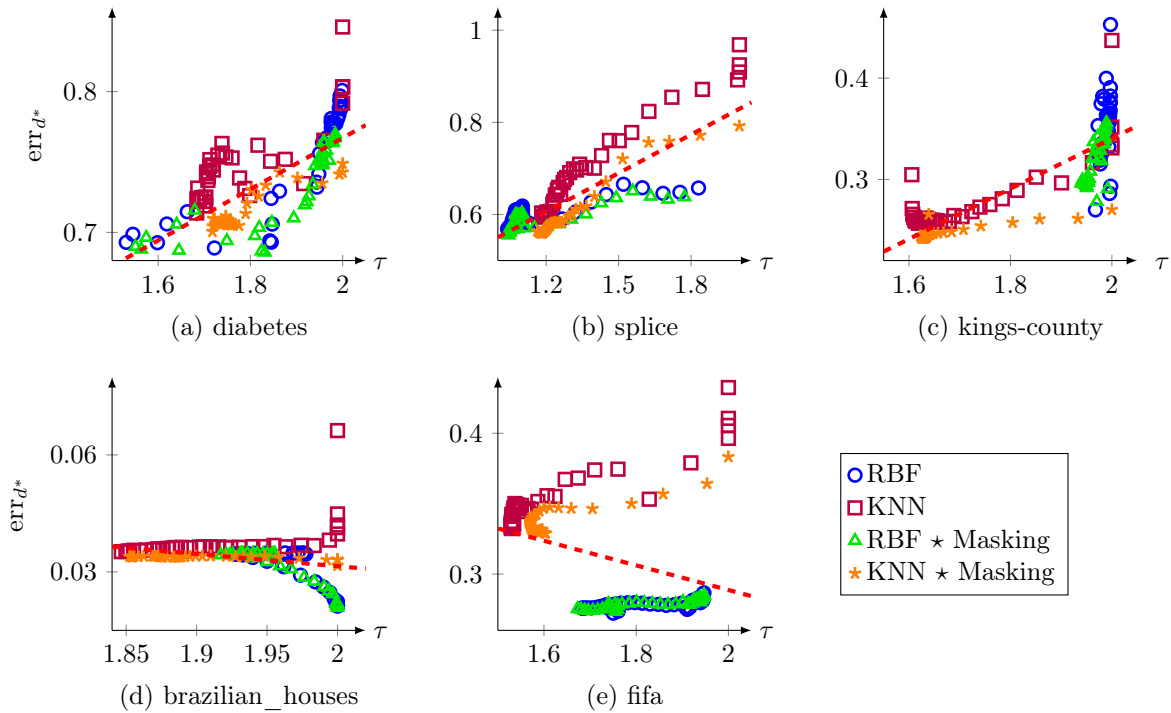
\begin{figure}[t]
    \centering
    \begin{tikzpicture}
\begin{axis}[
name=plot1,
clip marker paths=true,
 height = .30\linewidth,
    width = .30\linewidth,
    axis lines=left,  
    axis line style={-latex},  
    xlabel style = {
        at={(1.06,0.05)},
        anchor=north,
    },
title style={at={(0.5,-0.24)}, anchor=north},
title={(a) diabetes},
    xmin=1.5, xmax=2.05,
    ymin=0.68, ymax=0.86,
    xtick distance = 0.2,
    ytick distance = 0.1,
    xlabel={$\tau$},
    ylabel={$\err_{d^*}$},
    legend pos=north west,
    scatter/classes={
        1={mark=o,blue,line width=1pt,mark size=2.4pt},   
        2={mark=square,purple,line width=1pt,mark size=2.4pt}, 
        3={mark=triangle,green!90!black,line width=1pt,mark size=2.4pt}, 
        4={mark=star,orange,line width=1pt,mark size=2.4pt}
    }
]

\addplot[scatter, only marks, scatter src=explicit symbolic] 
    table[col sep=comma, x=tau, y=mse, meta=type] {data/diabetes.csv};

    \addplot[red, ultra thick, dashed, domain=1.5:2.05] {0.18192012*x + 0.40324192};
\end{axis}

\begin{axis}[
name=plot2,
at={($(plot1.south east)+(50,0)$)},
clip marker paths=true,
 height = .30\linewidth,
    width = .30\linewidth,
    axis lines=left,  
    axis line style={-latex},  
    xlabel style = {
        at={(1.06,0.05)},
        anchor=north,
    },
title style={at={(0.5,-0.24)}, anchor=north},
title={(b) splice},
    xmin=1.0, xmax=2.05,
    ymin=0.5, ymax=1.05,
    xtick distance = 0.3,
    ytick distance = 0.2,
    xlabel={$\tau$},
    legend pos=north west,
    scatter/classes={
        1={mark=o,blue,line width=1pt,mark size=2.4pt},   
        2={mark=square,purple,line width=1pt,mark size=2.4pt}, 
        3={mark=triangle,green!90!black,line width=1pt,mark size=2.4pt}, 
        4={mark=star,orange,line width=1pt,mark size=2.4pt}
    }
]

\addplot[scatter, only marks, scatter src=explicit symbolic] 
    table[col sep=comma, x=tau, y=mse, meta=type] {data/splice.csv};

    \addplot[red, ultra thick, dashed, domain=1.0:2.05] {0.27877287*x + 0.27161478};
\end{axis}

\begin{axis}[
name=plot3,
at={($(plot2.south east)+(50,0)$)},
clip marker paths=true,
 height = .30\linewidth,
    width = .30\linewidth,
    axis lines=left,  
    axis line style={-latex},  
    xlabel style = {
        at={(1.1,0.05)},
        anchor=north,
    },
title style={at={(0.5,-0.24)}, anchor=north},
title={(c) kings-county},
    xmin=1.55, xmax=2.05,
    ymin=0.22, ymax=0.47,
    xtick distance = 0.2,
    ytick distance = 0.1,
    xlabel={$\tau$},
    legend pos=north west,
    scatter/classes={
        1={mark=o,blue,line width=1pt,mark size=2.4pt},   
        2={mark=square,purple,line width=1pt,mark size=2.4pt}, 
        3={mark=triangle,green!90!black,line width=1pt,mark size=2.4pt}, 
        4={mark=star,orange,line width=1pt,mark size=2.4pt}
    }
]

\addplot[scatter, only marks, scatter src=explicit symbolic] 
    table[col sep=comma, x=tau, y=mse, meta=type] {data/kings-county.csv};

    \addplot[red, ultra thick, dashed, domain=1.55:2.05] {0.24798084*x -0.15551522};
\end{axis}

\begin{axis}[
name=plot4,
at={($(plot1.south west)+(0,-140)$)},
clip marker paths=true,
 height = .30\linewidth,
    width = .30\linewidth,
    axis lines=left,  
    axis line style={-latex},  
    ylabel style={
        at={(-0.3,0.5)}
    },
    xlabel style = {
        at={(1.06,0.05)},
        anchor=north,
    },
scaled ticks=false, 
/pgf/number format/fixed   ,
title style={at={(0.5,-0.24)}, anchor=north},
title={(d) brazilian\_houses},
    xmin=1.84, xmax=2.02,
    ymin=0.015, ymax=0.08,
    xtick distance = 0.05,
    ytick distance = 0.03,
    xlabel={$\tau$},
    ylabel={$\err_{d^*}$},
    legend pos=north west,
    scatter/classes={
        1={mark=o,blue,line width=1pt,mark size=2.4pt},   
        2={mark=square,purple,line width=1pt,mark size=2.4pt}, 
        3={mark=triangle,green!90!black,line width=1pt,mark size=2.4pt}, 
        4={mark=star,orange,line width=1pt,mark size=2.4pt}
    }
]

\addplot[scatter, only marks, scatter src=explicit symbolic] 
    table[col sep=comma, x=tau, y=mse, meta=type] {data/brazilian.csv};

    \addplot[red, ultra thick, dashed, domain=1.84:2.02] {-0.031479845*x + 0.09441593};
\end{axis}

\begin{axis}[
name=plot5,
at={($(plot4.south east)+(50,0)$)},
clip marker paths=true,
 height = .30\linewidth,
    width = .30\linewidth,
    axis lines=left,  
    axis line style={-latex},  
    xlabel style = {
        at={(1.06,0.05)},
        anchor=north,
    },
title style={at={(0.5,-0.24)}, anchor=north},
title={(e) fifa},
    xmin=1.5, xmax=2.05,
    ymin=0.26, ymax=0.44,
    xtick distance = 0.2,
    ytick distance = 0.1,
    xlabel={$\tau$},
    legend style={
        at={(1.35,0.7)},
        anchor=north west,
        cells={anchor=west},
    },
    scatter/classes={
        1={mark=o,blue,line width=1pt,mark size=2.4pt},   
        2={mark=square,purple,line width=1pt,mark size=2.4pt}, 
        3={mark=triangle,green!90!black,line width=1pt,mark size=2.4pt}, 
        4={mark=star,orange,line width=1pt,mark size=2.4pt}
    }
]

\addplot[scatter, only marks, scatter src=explicit symbolic] 
    table[col sep=comma, x=tau, y=mse, meta=type] {data/fifa.csv};
\legend{RBF,KNN,RBF $\star$ Masking,KNN $\star$ Masking};
    \addplot[red, ultra thick, dashed, domain=1.5:2.05] {-0.08744848*x + 0.46356109};
\end{axis}

\end{tikzpicture}
    \caption{Scatter plots of $\tau$ versus $\err_{d^*}$. Dashed line: Linear fit.}
    \label{fig:usefulness-metric}
\end{figure}

To understand when our metric might fail, we further visualize the results by plotting $\tau$ against $\err_{d^*}$ on five of the datasets in \Cref{fig:usefulness-metric}.
In this figure, plots (a), (b), and (c) are three success cases where a clear positive correlation can be observed, and plots (d) and (e) display two failure cases.
Plot (d) shows a common failure case: if $\tau$ is very close to $2 = \beta + 1$, meaning that the metric believes that the association is extremely weak or extremely strong, then the metric will predict that the context is bad.
However, a generally bad context can still be good on some tasks.
For example, a very weak context still works well on a task that only uses the top-$5$ singular functions of the context, as we have shown in \Cref{sec:metric}.
Therefore, it is advisable to abstain from using the metric when it is too close to $\beta + 1$.

Plot (e) shows a case where the metric is generally good for every single context type but has poor cross-type behavior. Specifically, it fails to predict that KNN is worse than RBF on this dataset.
This suggests that the metric might not be able to compare different types of contexts. For example, if two contexts of completely different types have similar spectra, then the metric will indicate that they are similarly useful. This is because the metric only depends on the spectrum.
However, it could be possible that for a particular task, one context is good and the other is bad, and our metric cannot reflect this disparity.

Overall, although there does not exist a universal metric that works for all contexts and tasks, and our metric does have failure cases, the experiment results here provide empirical evidence that more often than not, the proposed metric correlates well with the actual prediction error of the downstream linear probe. Hence, the proposed metric is useful for choosing hyperparameters and comparing contexts in practice.

\section{Conclusion}
\label{sec:conclusion}

The contexture theory established in this paper characterizes the mechanism of representation learning by clarifying the target of representation learning---the top singular functions of the expectation operator induced by the contexture, that is the association between the input $X$ and a context variable $A$. 
A representation that learns the contexture achieves the lowest worst-case approximation error on the class of tasks compatible with the context.
We show that most representation learning approaches could be cast as learning the contexture, and empirically demonstrate that the representations learned by large neural networks are highly aligned with the top singular functions.
We further analyze how to evaluate the usefulness of a context using its spectrum, and the key takeaway is that a good context should have a moderate association between $X$ and $A$.

Our analysis has three limitations, which lead to three open problems. First, our analysis focused on the minimizers of the objectives. However, \cite{cohen2021gradient} showed that deep models trained by popular gradient methods do not find the minimizers, but instead oscillate around the \emph{edge of stability}.
The open problem is whether the representation is always aligned with the top-$d$ singular functions as it is oscillating.
Second, we did not discuss the impact of the \textit{inductive bias} of the model architecture, such as the translation invariance of convolutional neural networks. Such inductive biases can affect the context and, therefore, the encoder. We pose how to integrate the effect of these biases into our theory as an open problem.
Third, our theory assumes that $\px$ is fixed. In practice, however, there is always a data distribution shift from pretraining to downstream.
Hence, the open problem is how to refine our theory to handle such distribution shifts.

\subsection*{Acknowledgements}
We thank Rattana Pukdee, Hugo Contant, Chenhao Zhang and Zihao Ye for their feedback on this paper.
Kai Yang did this work at Carnegie Mellon University, under the PKU-CMU undergraduate research program.
We acknowledge the support of NSF via IIS-2211907, ONR via N00014-23-1-2368, AFRL via FA8750-23-2-1015, and DARPA via HR00112020006.


\bibliographystyle{alpha}
\newcommand{\etalchar}[1]{$^{#1}$}


\clearpage
\appendix

\section{Proofs for Section \ref{sec:objectives}}

\subsection{Proof of Theorem \ref{thm:thm-supervised-cls}}
\label{app:proof-thm-supervised-cls}

\paragraph{Theorem \ref{thm:thm-supervised-cls}.}
\textit{
Let $A$ be a one-hot random vector.
Suppose the linear layer is unbiased, that is $\vb = \vzero$.
Then, $\Phi^*$ minimizes $\gR(\Phi)$ if and only if it extracts the top-$d$ eigenspace of $ \tp \Lambda \tpstar$, where $\kp(a,a') = \sI [a=a']$, or $(\Lambda g)(a) = g(a) \pa(a)$.
If all classes have the same size, then the top-$d$ eigenfunctions of $ \tp \Lambda \tpstar$ and $\tp \tpstar$ are the same.}

The following lemma will be very useful in the proof.

\begin{lemma}
\label{lem:lem-supervised-kern}
    $\tp \Lambda \tpstar$ is the integral kernel operator of the following kernel
    \begin{equation*}
        k(x, x') =  \iint \kp(a, a') P^+(a| x) P^+(a'| x') da da'.
    \end{equation*}
\end{lemma}
\begin{proof}
By definition, we have
\begin{equation*}
    (\tpstar h)(a') = \int h(x') P^+(x'|a')dx'.
\end{equation*}
Thus, we have
\begin{equation*}
\begin{aligned}
    (\Lambda \tpstar h)(a) &= \int (\tpstar h)(a') \kp(a, a') \pa(a') da'\\ 
    &= \iint h(x') P^+(x'|a') \kp(a,a') \pa(a') dx' da' \\
    &= \iint h(x') P^+(a'|x') \kp(a,a') \px(x') dx' da'.
\end{aligned}
\end{equation*}
This implies that
\begin{equation*}
\begin{aligned}
    ( \tp \Lambda \tpstar h)(x) &=  \int (\Lambda \tpstar h)(a) P^+(a|x) da \\
    &= \iiint h(x')  \kp(a,a') P^+(a|x) P^+(a'|x') \px(x')  da da' dx' \\
    &= \int h(x') k(x,x') \px(x') dx',
\end{aligned}
\end{equation*}
as desired.
\end{proof}

Then, we finish the proof of \Cref{thm:thm-supervised-cls}.

\begin{proof}
For any fixed $\Phi$, define
\begin{equation*}
    \gR(\Phi, \mW)  = \E_{P^+} \brac{ \norm{A - \mW \Phi(X)}_2^2 } = 
    \underset{X \sim \px}{\E} \; \underset{A \sim P^+(\cdot | X)}{\E} \; \brac{   \norm{A - \mW \Phi(X)}_2^2 } .
\end{equation*}
Assuming, without loss of generality, that $\E_{X \sim \px} [ \Phi_i  \Phi_j] = \delta_{ij}$; otherwise one can perform Gram-Schmidt process on $\Phi_i$ and change the value of $\mW$ respectively. Thus, it amounts to minimizing 

\begin{equation*}
\begin{aligned}
    \gR(\Phi, \mW) &=  \underset{X \sim \px}{\E} \; \underset{A \sim P^+(\cdot | X)}{\E} \; \brac{   \norm{ A - \mW  \Phi(X) }_2^2 } \\
    &= \underset{X \sim \px}{\E} \norm{\mW \Phi(X)}_2^2 - 2 \underset{X \sim \px}{\E} \; \underset{A \sim P^+(\cdot | X)}{\E} \; \dotp{ A, \mW  \Phi(X)} + \underset{A \sim \pa}{\E} \norm{A}_2^2 \\
    &= \norm{\mW}_F^2 - 2 \underset{X \sim \px}{\E} \; \underset{A \sim P^+(\cdot | X)}{\E} \; \dotp{ A, \mW  \Phi(X)} + \underset{A \sim \pa}{\E} \norm{ A}_2^2.
\end{aligned}
\end{equation*}
Denote $\mW = (w_{ij})_{1\leq i\leq d_A, 1\leq j \leq d}$. We have 
\begin{equation*}
    \frac{\partial  \gR}{\partial w_{ij}} = 2 w_{ij} - 2 \underset{X \sim \px}{\E} \; \underset{A \sim P^+(\cdot | X)}{\E} \; \brac{ A_i  \Phi_j(X)}  ,
\end{equation*}
which implies that for a fixed $\Phi$, the optimal $\mW$ that minimizes this loss should satisfy
\begin{equation*}
    w_{ij} = \underset{X \sim \px}{\E} \; \underset{A \sim P^+(\cdot | X)}{\E} \; \brac{ A_i  \Phi_j(X)}  .
\end{equation*}
Combining the minimizer of $\mW$ with $ \gR$ and notice that $\E_{A \sim \pa} \norm{ A}_2^2$ is a constant, it suffices to \textbf{maximize}
\begin{equation*}
\begin{aligned}
    F(\Phi) &= \sum_{i,j} \brac{\underset{X \sim \px}{\E} \; \underset{A \sim P^+(\cdot | X)}{\E} \;  A_i  \Phi_j(X)}^2 \\
    &= \int \sum_j  \Phi_j(x_1)  \Phi_j(x_2) \dotp{ a_1,  a_2} \px(x_1) P^+(a_1 | x_1) \px(x_2) P^+(a_2 | x_2) dx_1 da_1 dx_2 da_2 \\
    &= \iint \sum_j  \Phi_j(x_1)  \Phi_j(x_2) \hat k(x_1, x_2) \px(x_1)  \px(x_2) dx_1 dx_2,
\end{aligned}
\end{equation*}
where 
\begin{align}
    \hat k(x_1, x_2)  &= \iint \dotp{ a_1,  a_2} P^+(a_1 | x_1) P^+(a_2 | x_2) da_1 da_2 \label{eqn:supervised-unbiased-kernel} \\
    &= \iint \sI [a_1 = a_2] P^+(a_1 | x_1) P^+(a_2 | x_2) da_1 da_2. \nonumber
\end{align}
Thus $\Phi^*$ is a minimizer of $\gR(\Phi)$ if $\Phi^*$ extracts the top-$d$ eigenfunctions of $\hat k(x_1, x_2)$. Combining with \Cref{lem:lem-supervised-kern} yields that $k_\Lambda(a, a') = \sI[a = a']$.
Furthermore, we have $(\Lambda g)(a) = \int g(a') k_\Lambda(a, a') d\pa(a') = g(a)\pa(a)$, as desired.

If all classes have the same size, we have $\pa(a) \equiv c \in (0, 1)$ where $c$ is a constant. Thus $(\Lambda g)(a) = g(a)\pa(a) = cg(a)$, which implies that $ \tp \Lambda \tpstar = c \tp \tpstar$. This concludes that $\tp \Lambda \tpstar$ and $\tp \tpstar$ share the same top-$d$ eigenfunctions.
\end{proof}

\subsection{Proof of Theorem \ref{thm:thm-supervised-bal-cls}}
\label{app:proof-thm-supervised-bal-cls}

\paragraph{Theorem \ref{thm:thm-supervised-bal-cls}.}
\textit{Under the setting of \Cref{thm:thm-supervised-cls}, suppose the linear layer is biased. Then, $\Phi^*$ minimizes $\gR_{\rm bal}(\Phi)$ if and only if it learns the contexture of $P^+$.}

\begin{proof}
Let us first assume that the linear layer is unbiased, that is $\vb = \vzero$.
For any fixed $\Phi$, define
\begin{equation*}
    \gR(\Phi, \mW)  = \E_{P^+} \brac{ \frac{1}{\sqrt{\pa (A)}}\norm{A - \mW \Phi(X)}_2^2 } = 
    \underset{X \sim \px}{\E} \; \underset{A \sim P^+(\cdot | X)}{\E} \; \brac{   \frac{1}{\sqrt{\pa (A)}}\norm{A - \mW \Phi(X)}_2^2 } .
\end{equation*}
Assuming, without loss of generality, 
\begin{equation*}
    \underset{X \sim \px}{\E} \; \underset{A \sim P^+(\cdot | X)}{\E} \;  \brac{ \frac{1}{\sqrt{\pa (A)}} \Phi_i \Phi_j } = \delta_{ij};
\end{equation*}
otherwise we can perform Gram-Schmidt process on $\Phi_i$ and change the value of $\mW$ respectively. Thus it suffices to minimize 

\begin{equation*}
\begin{aligned}
    \gR(\Phi, \mW) &=  \underset{X \sim \px}{\E} \; \underset{A \sim P^+(\cdot | X)}{\E} \; \brac{  \frac{1}{\sqrt{\pa (A)}} \norm{ A - \mW  \Phi(X) }_2^2 } \\
    &= \underset{X \sim \px}{\E} \; \underset{A \sim P^+(\cdot | X)}{\E} \; \brac{  \frac{1}{\sqrt{\pa (A)}} \norm{ \mW  \Phi(X) }_2^2 } \\
    & \quad - 2 \underset{X \sim \px}{\E} \; \underset{A \sim P^+(\cdot | X)}{\E} \;  \dotp{ \frac{A}{\sqrt{\pa (A)}}, \mW  \Phi(X)} + \underset{A \sim \pa}{\E} \brac{\frac{ \norm{ A}_2^2}{\sqrt{\pa (A)}}} \\
    &= \norm{\mW}_F^2 - 2 \underset{X \sim \px}{\E} \; \underset{A \sim P^+(\cdot | X)}{\E} \;  \dotp{ \frac{A}{\sqrt{\pa (A)}}, \mW  \Phi(X)} + \underset{A \sim \pa}{\E} \brac{\frac{ \norm{ A}_2^2}{\sqrt{\pa (A)}}}  .
\end{aligned}
\end{equation*}
Denote $\mW = (w_{ij})_{1\leq i\leq d_A, 1\leq j \leq d}$. We have 
\begin{equation*}
    \frac{\partial  \gR}{\partial w_{ij}} = 2 w_{ij} - 2 \underset{X \sim \px}{\E} \; \underset{A \sim P^+(\cdot | X)}{\E} \; \brac{  \frac{A_i}{\sqrt{\pa (A)}} \Phi_j(X)}  ,
\end{equation*}
which implies that for a fixed $\Phi$, the minimizer of $\mW$ satisfies 
\begin{equation*}
    w_{ij} = \underset{X \sim \px}{\E} \; \underset{A \sim P^+(\cdot | X)}{\E} \; \brac{  \frac{A_i}{\sqrt{\pa (A)}} \Phi_j(X)}  .
\end{equation*}
Combining the minimizer of $\mW$ with $ \gR$, it suffices to maximize
\begin{equation*}
    \gR' = \sum_{i,j} \brac{\underset{X \sim \px}{\E} \; \underset{A \sim P^+(\cdot | X)}{\E} \;  \frac{A_i}{\sqrt{\pa (A)}} \Phi_j(X)  }^2 
    = \iint \sum_j  \Phi_j(x_1)  \Phi_j(x_2) \hat k(x_1, x_2) \px(x_1)  \px(x_2) dx_1 dx_2,
\end{equation*}
where 
\begin{equation*}
\begin{aligned}
    \hat k(x_1, x_2) &= \iint \frac{\dotp{ a_1,  a_2}}{\sqrt{\pa(a_1) \pa(a_2)}} P^+(a_1 | x_1) P^+(a_2 | x_2) da_1 da_2 \\
    & = \iint \frac{\sI[a_1 = a_2]}{\sqrt{\pa(a_1) \pa(a_2)}} P^+(a_1 | x_1) P^+(a_2 | x_2) da_1 da_2 \\
    & = \int \frac{P^+(a | x_1) P^+(a | x_2)}{\pa(a)} dy.
\end{aligned}
\end{equation*}
Thus, $\Phi^*$ is a minimizer of $\gR(\Phi)$ if $\Phi^*$ extracts the top-$d$ eigenfunctions of $\hat k(x_1, x_2)$. Note that here the top-$d$ eigenfunctions include $\mu_0 \equiv 1$.
When we include the bias $\vb$ in the linear layer, then this $\mu_0$ is covered by the bias term, so $\Phi$ will extract $\mu_1,\cdots,\mu_d$.
\end{proof}

\subsection{Proof of Theorem \ref{thm:obj-regression}}
\label{app:proof-thm-obj-regression}

\paragraph{Theorem \ref{thm:obj-regression}.}
\textit{$\Phi^*$ minimizes \Cref{eqn:supervised-obj} if and only if $\Phi^*$ extracts the top-$d$ eigenspace of $\tp \Lambda \tpstar$.
If the linear layer is unbiased ($\vb = \vzero$), then $\kp(a,a') = \dotp{a, a'}$;
if it is biased ($\vb$ can be arbitrary), then $\kp(a,a') = \dotp{\tilde{a}, \tilde{a'}}$.}

\begin{proof}
For the unbiased linear model, the proof is similar to that of \Cref{thm:thm-supervised-cls}. Combining \Cref{eqn:supervised-unbiased-kernel} and \Cref{lem:lem-supervised-kern} yields the desired result.

Next, we consider a biased linear model. For a variable $z$, we denote $\bar{z} = \E[Z]$, and $\tilde{z} = z - \E[Z]$ as its centered version.

For any fixed $\Phi$, define
\begin{equation*}
\begin{aligned}
    \gR(\Phi, \mW, \vb) & =  \underset{X \sim \px}{\E} \; \underset{A \sim P^+(\cdot | X)}{\E} \; \brac{   \norm{A - \mW \Phi(X) - \vb}_2^2 }  \\
    & =  \underset{X \sim \px}{\E} \; \underset{A \sim P^+(\cdot | X)}{\E} \; \brac{   \norm{A - \mW \Phi(X) - \vb}_2^2 }  \\
    & =  \underset{X \sim \px}{\E} \; \underset{A \sim P^+(\cdot | X)}{\E} \; \brac{   \norm{\tilde A - \mW \tilde \Phi(X) - \hat \vb}_2^2 }  \\
    & = \underset{X \sim \px}{\E} \; \underset{A \sim P^+(\cdot | X)}{\E} \; \brac{   \norm{\tilde A - \mW \tilde \Phi(X) }_2^2 } + \norm{\hat \vb}_2^2 
\end{aligned}
\end{equation*}
where $\hat \vb = \mW \E_{X \sim \px}[\Phi(X)] - \E_{A \sim \pa}[A] + \vb$. Thus, for any fixed $\Phi, \mW$, the optimal $\vb = \E_{A \sim \pa}[A] - \mW \E_{X \sim \px}[\Phi(X)]$.

Assuming, without loss of generality, $\E_{X \sim \px} [\tilde \Phi_i \tilde \Phi_j] = \delta_{ij}$; otherwise we can perform Gram-Schmidt process on $\tilde \Phi_i$ and change the value of $\mW$ respectively. Thus, it suffices to minimize 

\begin{equation*}
\begin{aligned}
    \hat \gR(\Phi, \mW) &=  \underset{X \sim \px}{\E} \; \underset{A \sim P^+(\cdot | X)}{\E} \; \brac{   \norm{\tilde A - \mW \tilde \Phi(X) }_2^2 } \\
    &= \underset{X \sim \px}{\E} \norm{\mW \tilde \Phi(X)}_2^2 - 2 \underset{X \sim \px}{\E} \; \underset{A \sim P^+(\cdot | X)}{\E} \; \dotp{\tilde Y, \mW \tilde \Phi(X)} + \underset{A \sim \pa}{\E} \norm{\tilde A}_2^2 \\
    &= \norm{\mW}_F^2 - 2 \underset{X \sim \px}{\E} \; \underset{A \sim P^+(\cdot | X)}{\E} \; \dotp{\tilde A, \mW \tilde \Phi(X)} + \underset{A \sim \pa}{\E} \norm{\tilde A}_2^2.
\end{aligned}
\end{equation*}

Denote $\mW = (w_{ij})_{1\leq i\leq d_y, 1\leq j \leq d}$. We have 
\begin{equation*}
    \frac{\partial \hat \gR}{\partial w_{ij}} = 2 w_{ij} - 2 \underset{X \sim \px}{\E} \; \underset{A \sim P^+(\cdot | X)}{\E} \; \brac{\tilde A_i \tilde \Phi_j(X)}  ,
\end{equation*}
which implies that for a fixed $\Phi$, the minimizer of $\mW$ satisfies 
\begin{equation*}
    w_{ij} = \underset{X \sim \px}{\E} \; \underset{A \sim P^+(\cdot | X)}{\E} \; \brac{\tilde A_i \tilde \Phi_j(X)}  .
\end{equation*}
Combining the minimizer of $\mW$ with $\hat \gR$ and notice that $\E_{A \sim \pa} \norm{\tilde A}_2^2$ is a constant, it suffices to maximize
\begin{equation*}
\begin{aligned}
    \hat \gR' &= \sum_{i,j} \brac{\underset{X \sim \px}{\E} \; \underset{A \sim P^+(\cdot | X)}{\E} \; \tilde A_i \tilde \Phi_j(X)}^2 \\
    &= \int \sum_j \tilde \Phi_j(x_1) \tilde \Phi_j(x_2) \dotp{\tilde a_1, \tilde a_2} \px(x_1) P^+(a_1 | x_1) \px(x_2) P^+(a_2 | x_2) dx_1 da_1 dx_2 da_2 \\
    &= \iint \sum_j \tilde \Phi_j(x_1) \tilde \Phi_j(x_2) \hat k(x_1, x_2) \px(x_1)  \px(x_2) dx_1 dx_2,
\end{aligned}
\end{equation*}
where 
\begin{equation*}
    \hat k(x_1, x_2) = \iint \dotp{\tilde a_1, \tilde a_2} P^+(a_1 | x_1) P^+(a_2 | x_2) da_1 da_2.
\end{equation*}
Notice that
\begin{equation*}
\begin{aligned}
    \iint \hat k(x_1, x_2) \px(x_1)  \px(x_2) dx_1 dx_2 &= \int \dotp{\tilde a_1, \tilde a_2} P^+(x_1, a_1) P^+(x_2 , a_2) dx_1 da_1 dx_2 da_2     \\
    & = \int \dotp{\tilde a_1, \tilde a_2} \pa(a_1) \pa(a_2) da_1 da_2  = 0,
\end{aligned}
\end{equation*}
thus $\Phi^*$ is a minimizer of $\gR(\Phi)$ if $\tilde \Phi^*$ extracts the top-$d$ eigenfunctions of $\hat k(x_1, x_2)$. Combining with \Cref{lem:lem-supervised-kern} yields the desired results.
\end{proof}

\subsection{Proof of Theorem \ref{thm:ssl-singular}}
\label{app:proof-thm-ssl-singular}

\paragraph{Theorem \ref{thm:ssl-singular}.}
\textit{$\Psi^*$ minimizes $\gL_{\rm{C}}$ or $\gL_{\rm{N}}$ if and only if $\tPhi^* = \tp \tPsi^*$ learns the contexture.}

\begin{proof}
    We prove the two cases as follows.
    \begin{enumerate}[label=(\roman*)]
    \item The spectral contrastive loss is 
    \begin{equation*}
    \gL_{\rm{C}}(\Psi) = \underset{X \sim \px}{\E} \; \underset{A,A^+ \sim P^+(\cdot |X)}{\E} \brac{ - \dotp{\tPsi(A), \tPsi(A^+)} + \frac{1}{2} \underset{A^- \sim \pa}{\E} \brac{\dotp{\tPsi(A), \tPsi(A^-)} ^2} }   .
    \end{equation*}
    Suppose $\psi_i = \sum_{j\geq 0} c_{ij}\nu_j$ where $\nu_j$ is the ONB of $L^2(\pa)$ in \Cref{lem:duality}. Since $\nu_j$ is the ONB of $L^2(\pa)$ and $\nu_0 \equiv 1$, we can get for $j\geq 1$, $\E_{\pa} [\nu_j(a)] = \delta_{0,j} = 0$. Thus we can get $\tilde\psi_i = \psi_i - \E[\psi_i] = \sum_{j\geq 1} c_{ij}\nu_j$.
    
    Denote matrix $\mC = (c_{ij})_{1\leq i\leq d,j\geq 1}$, matrix $\mB = (b_{ij}) := \mC^\top \mC$, and matrix $\mD = \diag(s_1^2,s_2^2,\cdots)$ where $s_i$ is the singular value of $\tp$. We have 
    \begin{equation*}
    \begin{aligned}
        & \quad \underset{X \sim \px}{\E} \underset{A,A^+ \sim P^+(\cdot |X)}{\E} \brac{\dotp{\tPsi(A), \tPsi(A^+)}} \\
        &= \iiint \dotp{\tPsi(a), \tPsi(a^+)} P^+(a|x) P^+(a^+|x) \px(x) dxdada^+ \\
        &= \int \dotp{\int \tPsi(a) P^+(a|x) dy, \int \tPsi(a^+) P^+(a^+|x) da^+} p(x) dx\\
        &= \int \dotp{\tp \tPsi(x), \tp \tPsi(x)} p(x) dx = \|\tp \tPsi\|^2_{\px} \\
        &= \sum_i s_i^2 b_{ii}   ;
    \end{aligned}
    \end{equation*}
    and 
    \begin{equation*}
    \begin{aligned}
        \underset{A,A^- \sim \pa}{\E}  \brac{\dotp{\tPsi(A), \tPsi(A^-)} ^2}
        &= \iint \brac{\sum_{i=1}^d \tilde \psi_i(a)\tilde \psi_i(a^-)}^2 d\pa(a) d\pa(a^-)  \\
        &= \sum_{1\leq i,j\leq d} \brac{\int \tilde\psi_i(a)\tilde\psi_j(a) d\pa(a)}^2 \\
        &= \sum_{i,j} b_{ij}^2   .
    \end{aligned}
    \end{equation*}
    Thus, we have
    \begin{equation*}
    \gL_{\rm{C}}(\Psi) = - \sum_i s_i^2 b_{ii} + \frac12 \sum_{i,j} b_{ij}^2 \\= \|\mB - \mD\|_F^2 - \|\mD\|_F^2  .
    \end{equation*}
    So if suffices to minimize $\|\mB - \mD\|_F^2$ where $\text{rank}(\mB) \leq d$. By Eckart-Young-Mirsky Theorem, we know the minimizer of $\mB$ is $\mB^* = \text{diag}(s_1^2,\cdots,s_d^2)$. Thus, the minimizer of $\mC$ should be $\mC^* = \mU\diag(s_1,\cdots,s_d)$ where $\mU \in \sR^{d\times d}$ is an orthonormal matrix. This indicates the minimizer $\tPsi^*$ extracts the top-$d$ singular functions of $\tp$, and $\tPhi^*$ learns the contexture of $P^+$.
    
    \item The non-contrastive loss is 
    \begin{equation*}
    \begin{aligned}
        & \gL_{\rm{N}}(\Psi) = \underset{X \sim \px}{\E} \; \underset{A,A^+ \sim P^+(\cdot |X)}{\E} \brac{-\dotp{\tPsi(A), \tPsi(A^+)}} ; \\ 
        & \gL_{\rm{N}}'(\Psi) = \underset{X \sim \px}{\E} \; \underset{A,A^+ \sim P^+(\cdot |X)}{\E} \brac{ \norm{\Psi(A) - \Psi(A^+)}_2^2 }  , 
    \end{aligned}
    \end{equation*}
    where $\Cov_{\pa}\brac{\Psi} = \mI$. Since for any $\Psi$,
    \begin{equation*}
        \gL_{\rm{N}}'(\Psi) - \gL_N(\Psi) = 2 \underset{A\sim \pa}{\E} \brac{ \norm{\tPsi(A)}_2^2 } = 2d
    \end{equation*}
    is a constant, thus $\Psi^*$ minimizes $\gL_{\rm{N}}(\Psi) \iff \Psi^*$ minimizes $\gL_{\rm{N}}'(\Psi)$. 
    
    Suppose $\psi_i = \sum_{j\geq 0} c_{ij}\nu_j$ where $\nu_j$ is the ONB of $L^2(\pa)$ in \Cref{lem:spectral-decomposition}. Since $\E_{\pa} [\nu_j(a)] = \delta_{0,j}$, we can get $\tilde\psi_i = \psi_i - \E[\psi_i] = \sum_{j\geq 1} c_{ij}\nu_j$. 
    
    We now consider the minimizer of $\gL_{\rm{N}}(\Psi)$. By the calculation in (i), we obtain
    \begin{equation*}
        \gL_{\rm{N}}(\Psi) = -\underset{X \sim \px}{\E} \underset{A,A^+ \sim P^+(\cdot |X)}{\E} \brac{\dotp{\tPsi(A), \tPsi(A^+)}} = -\|\tp \tPsi\|^2_{\px} = - \sum_i s_i^2 b_{ii}   .
    \end{equation*}
    By $\E_{\pa}\brac{\tilde\psi_i \tilde\psi_j} = \delta_{ij}$, we have
    \begin{equation*}
        \sum_{i} b_{ii} = \sum_{i,j} c_{ij}^2 = d   .
    \end{equation*}
    Since $\nu_i$ is an ONB of $L^2(\pa)$, $\tilde\psi_1, \cdots, \tilde\psi_d$ are orthogonal, we have
    \begin{equation}
        b_{ii} = \sum_{j=1}^d c_{ji}^2 = \sum_{j=1}^d \dotp{\tilde \psi_j, \nu_i}_{\pa}^2 \leq \|\nu_i\|_{\pa}^2 = 1  .
        \label{eqn:proof-prop2-projection}
    \end{equation}
    Thus, we conclude that
    \begin{equation*}
        \gL_{\rm{N}}(\Psi) + \sum_{i=1}^d s_i^2 = \sum_{i=1}^d s_i^2(1 - b_{ii}) - \sum_{i>d} s_i^2 b_{ii} \geq \sum_{i=1}^d s_d^2 (1 - b_{ii}) - \sum_{i>d} s_d^2 b_{ii} = 0,
    \end{equation*}
    which implies that $\gL_{\rm{N}}(\Psi) \geq - \sum_{i=1}^d s_i^2$. To attain equality, we will have $b_{ii} = 1$ for $i=1,\cdots,d$, and $b_{ii} = 0$ for $i \geq d+1$. By \Cref{eqn:proof-prop2-projection}, we can know $\Psi^*$ extracts the span of $\nu_1,\cdots,\nu_d$, indicating that $\tPsi^*$ extracts the top-$d$ singular functions of $\tp$ and $\tPhi^*$ learns the contexture of $P^+$.
\end{enumerate}
\end{proof}

\subsection{Proof of Theorem \ref{thm:node-repre}}
\label{app:proof-thm-node-repre}

\paragraph{Theorem \ref{thm:node-repre}.}
\textit{Let $\Phi^*$ be any solution to \Cref{eqn:node-repre-obj} (so that for any constant $c$, $\Phi^* + c$ is also a solution).
Then, $\tPhi^*$ learns the contexture of $P^+$.}

\begin{proof}
Without loss of generality, suppose $\bPhi = \vzero$.
We have
\begin{equation*}
    \paren{\tp f} (u) = \sum_v f(v) \frac{w(u,v)}{d(u)}; \quad \dotp{\tp f, g}_{\px} = \sum_{u,v} f(u) g(v) \frac{w(u,v)}{d_{\textrm{sum}}} = \dotp{f, \tp g}_{\px},
\end{equation*}
which implies that $\tp$ is self-adjoint.
Therefore, the eigenfunctions of $\tp$ are the same as those of $\tpstar \tp$, with square root eigenvalues.

For the objective of \Cref{eqn:node-repre-obj}, we have
\begin{equation*}
\begin{aligned}
    \frac{1}{2} \E_{(u,v) \sim P_w} \brac{ \norm{\Phi(u) - \Phi(v)}_2^2 } & = \underset{(u,v) \sim P_w}{\E} \brac{\norm{\Phi(u)}_2^2 - \dotp{\Phi(u), \Phi(v)}} \\
    & = \sum_{i=1}^d \paren{ \norm{\phi_i}_\px^2 - \dotp{\phi_i, \tp \phi_i}_\px } \\
    & = d - \sum_{i=1}^d \dotp{\phi_i, \tp \phi_i}_\px .
\end{aligned}
\end{equation*}
Note that $(u,v)$ and $(v,u)$ can be drawn from $P_w$ with equal probability.
We conclude that $\Phi$ extracts the top-$d$ eigenfunctions of $\tp$, which are the same as the top-$d$ eigenfunctions of $\tpstar \tp$, or the top-$d$ singular functions of $\tp$. This implies that $\tPhi$ learns the contexture of $\tp$.
\end{proof}

\section{Proofs for Section \ref{sec:results}}

\subsection{Proof of Theorem \ref{thm:desiderata}}
\label{app:proof-desiderata}

\paragraph{Theorem \ref{thm:desiderata}.}
\textit{For any $f^* \in \fep$, there exists a $g^* \in \lap$ such that $f^*(x) = \E\sbrac{g^*(A)}{x}$, and $g^*$ satisfies}
\begin{equation}
\label{eqn:desiderata}
\underset{X \sim \px}{\E} \; \underset{A,A' \sim P^+\sparen{\cdot}{X}}{\E} \brac{ \paren{g^*(A) - g^*(A')}^2 } \le 4\epsilon \norm{g^*}_{\pa}^2  .
\end{equation}

\begin{proof}
    Let $g^* = \sum s_i u_i \nu_i$. We have already explained that if $f^* \in \fep$, i.e., $\E[f^*] = 0$ and $\rho(f^*, P^+) \geq 1 - \eps$, then it must satisfy the condition \wrt{} $g^*$:
    \begin{equation*}
        \frac{\dotp{f^*, \tp g^*}_{\px}}{\norm{f^*}_{\px} \norm{g^*}_{\pa}} \geq 1 - \eps .
    \end{equation*}
    For \Cref{eqn:desiderata}, we have $P(A' | A = a) = \int P^+(A'| X = x) P^+(x| A' = a) dx$, where $P^+(x|a) = \frac{P^+(a|x) \px(x)}{\pa(a)}$ by Bayes rule.
    Then, using \Cref{def:two-kernels} we have $P(A' | A=a) = \ka(a,a') \pa(a')$, which implies that
    \begin{align*}
    & \E_{X \sim \px} \E_{A,A' \sim P^+(\cdot |X)} \brac{g^*(A) g^*(A')} = \E_{A \sim \pa} \E_{A' \sim P(\cdot|A)} \brac{ g^*(A) g^*(A') } \\ 
     = \; & \E_A \brac{ g^*(A)  \int g^*(a') P(a' | A) da' }  = \E_A \brac{ g^*(A) \int g^*(a') \ka(a,a') \pa(a') da' }  = \dotp{ g^*, \tka g^* }_\pa .
    \end{align*}
    Since $\tka g^* = \tpstar \tp g^* = \sum s_i^3 u_i \nu_i$, \Cref{eqn:desiderata} is equivalent to $\sum (s_i^2-s_i^4) u_i^2 \le 2\epsilon \sum s_i^2 u_i^2$.
    Meanwhile, we have $\sum s_i^2 u_i^2 \ge (1-\epsilon)^2 \sum u_i^2 \ge (1-2\epsilon) \sum u_i^2$. By Cauchy-Schwarz inequality, we have $(\sum s_i^4 u_i^2) (\sum u_i^2) \ge (\sum s_i^2 u_i^2)^2 \ge (1-2\epsilon) (\sum u_i^2) (\sum s_i^2 u_i^2) $, which proves \Cref{eqn:desiderata}.
\end{proof}

\subsection{Proof of Theorem \ref{thm:top-d-optimal}}
\label{app:proof-top-d-optimal}

\paragraph{Theorem \ref{thm:top-d-optimal}.}
\textit{Suppose $1 - s_1 \leq \eps \leq 1 - \sqrt{\frac{s_1^2 + s_2^2}{2}}$. 
For any $d$, among all $\Phi = [\phi_1,\cdots,\phi_d]$ where $\phi_i \in \lxp$, $\Phi$ minimizes $\err(\Phi; \fep)$ if and only if it learns the contexture of $\tp$. The error is given by
\[
\min_{\Phi: \gX \rightarrow \R^d, \; \phi_i \in \lxp} \;  \err \paren{ \Phi; \gF_\epsilon(P^+) } = \frac{s_1^2 - (1 - \eps)^2}{s_1^2 - s_{d+1}^2} .
\]
Conversely, for any $d$-dimensional encoder $\Phi$ and any $\epsilon>0$, there exists $f \in \lxp$ such that $\rho(f,P^+) = 1-\epsilon$, and $\err(\Phi, f) \ge \frac{s_1^2 - (1 - \eps)^2}{s_1^2 - s_{d+1}^2}$.}

\begin{proof}

\textbf{Necessity:} Since $\sspan(\Phi)$ is at most rank-$d$, thus there exists $f_1 \in \sspan\{ \mu_1,\cdots, \mu_{d+1}\}$ with $\|f_1\|_\px = 1$ that is orthogonal to $\sspan(\Phi)$. Thus, there exists $f_1, f_2 \in \sspan\{ \mu_1, \cdots, \mu_{d+1} \}$ with $\|f_1\|_\px = \|f_2\|_\px = 1$, $f_1$ is orthogonal to $\sspan(\Phi)$ and $f_2 \in \sspan(\Phi)$ (thus $f_1 \perp f_2$), and $\mu_1 \in \sspan\{f_1, f_2\}$.  
Suppose $\mu_1 = \alpha_1 f_1 + \alpha_2 f_2$ (without loss of generosity, assuming $\alpha_1, \alpha_2 \in [0, 1]$) and denote $f_0 = \alpha_2 f_1 - \alpha_1 f_2$. Then $\| f_0 \|_\px = 1$ and $\langle \mu_1, f_0\rangle_\px = 0$. Since $f_1, f_2 \in \sspan\{\mu_1, \cdots, \mu_{d+1}\}$, we have $f_0 \in \sspan\{\mu_2, \cdots, \mu_{d+1}\}$ and thus $\E[f_0] = 0$.

Consider $f = \beta_1 \mu_1 + \beta_2 f_0$ where $\beta_1^2 + \beta_2^2 = 1$, $\beta_1,\beta_2 \in [0,1]$. 
Denote $f = \sum_{i\geq 1} u_i \mu_i$, we can get $\sum_i u_i^2 =1$ and
\begin{equation*}
    \beta_2^2 \leq \frac{s_1^2 - (1-\eps)^2}{s_1^2 - s_{d+1}^2}  \implies
    \sum_{i\geq 1} s_i^2 u_i^2 \geq s_1^2\beta_1^2 + s_{d+1}^2 \beta_2^2 = s_1^2 - (s_1^2 - s_{d+1}^2) \beta_2^2 \geq (1 - \eps)^2 \sum_i u_i^2    .
\end{equation*}
Obviously, $f \in \gF(P^+)$. We have
\begin{equation*}
    f = \beta_1 \mu_1 + \beta_2 f_0 = \beta_1 (\alpha_1 f_1 + \alpha_2 f_2) + \beta_2 (\alpha_2 f_1 - \alpha_1 f_2) = (\alpha_1 \beta_1 + \alpha_2 \beta_2) f_1 + (\alpha_2 \beta_1 - \alpha_1 \beta_2) f_2.
\end{equation*}
By the definition of $f_1,f_2$ we can know the approximation error for $f$ is $(\alpha_1 \beta_1 + \alpha_2 \beta_2)^2$.
We can show $F(\alpha_1) = \alpha_1 \beta_1 + \alpha_2 \beta_2 = \alpha_1 \beta_1 + \sqrt{1 - \alpha_1^2} \beta_2$ ($\alpha_1 \in [0,1]$) first increases then decreases when $\beta_1, \beta_2 \in [0,1]$. Thus $F(\alpha_1)^2 \geq \min\{ F(0)^2, F(1)^2\} = \min \{\beta_1^2, \beta_2^2\}$.
Take $\beta_2^2 =  \frac{s_1^2 - (1-\eps)^2}{s_1^2 - s_{d+1}^2} \leq \frac12$, we can get for $f$, the approximation error is always at least $\frac{s_1^2 - (1-\eps)^2}{s_1^2 - s_{d+1}^2}$. 

To attain equality, we must have $\sum_{i\geq 1} s_i^2 u_i^2 = s_1^2\beta_1^2 + s_{d+1}^2 \beta_2^2$. This implies that $f_1 = \mu_{d+1}$, indicating that $\sspan(\phi_1, \cdots, \phi_d) = \sspan(\mu_1, \cdots, \mu_d)$. Thus $\Phi$ learns the contexture of $\tp$.

Furthermore, denote $f_0 = k_2\mu_2 + \cdots + k_{d+1}\mu_{d+1}$ and consider $f = \beta_1 \mu_1 + \beta_2 f_0$ where $\beta_1^2 + \beta_2^2 = 1$, $\beta_1,\beta_2 \in [0,1]$. By the definition of $f_0$ and $f$, we have $\|f\|_\px = 1$ and $f = \beta_1 \mu_1 + \beta_2 \mu_2 = \beta_1 \mu_1 + \beta_2 k_2 \mu_2 + \cdots + \beta_2 k_{d+1} \mu_{d+1}$.
Thus
\begin{equation*}
    \rho^2(f, P^+) = s_1^2 \beta_1^2 + \beta_2^2 \sum_{i=2}^{d+1} s_i^2k_i^2 = s_1^2 - \left(s_1^2 - \sum_{i=2}^{d+1} s_i^2k_i^2\right) \beta_2^2.
\end{equation*}
Take
\begin{equation*}
    \beta_2^2 = \frac{s_1^2 - (1 - \eps)^2}{s_1^2 - \left( \sum_{i=2}^{d+1} s_i^2k_i^2 \right)} \leq \frac{s_1^2 - (1 - \eps)^2}{s_1^2 - s_2^2} \leq \frac{s_1^2 - \frac{s_1^2 + s_2^2}{2}}{s_1^2 - s_2^2} = \frac12,
\end{equation*}
we have $\rho(f, P^+) = 1 - \eps$. 
Similarly, the approximation error for $f$ is 
\begin{equation*}
    (\alpha_1 \beta_1 + \alpha_2 \beta_2)^2 \geq \min \{\beta_1^2, \beta_2^2\} = \beta_2^2 = \frac{s_1^2 - (1 - \eps)^2}{s_1^2 - \left( \sum_{i=2}^{d+1} s_i^2k_i^2 \right)}  \geq \frac{s_1^2 - (1 - \eps)^2}{s_1^2 - s_{d+1}^2} .
\end{equation*}

\textbf{Sufficiency:} For any $f \in \gF(P^+)$ with $\norm{f}_\px = 1$ and $\E[f] = 0$, denote $f = \sum_{i\geq 1} u_i \mu_i$ where $\sum_{i\geq 1} u_i^2 = 1$. Obviously we have $(1 - \eps)^2 \leq  \sum_{i\geq 1} s_i^2 u_i^2 \leq 1$. Notice that when $\sspan(\phi_1, \cdots, \phi_d) = \sspan(\mu_1, \cdots, \mu_d)$ since $\Phi$ learns the contexture of $\tp$, the approximation of $f$ will be $\sum_{i \geq d + 1} u_i^2 := A$. 
By the given conditions, we have 
\begin{equation*}
    (1 - \eps)^2 \leq  \sum_{i \geq 1} s_i^2 u_i^2 \leq s_1^2\sum_{i=1}^d u_i^2 + s_{d+1}^2\sum_{i \geq d+1} u_i^2 = s_1^2 - (s_1^2 - s_{d+1}^2) A   ,
\end{equation*}
and this implies that 
\begin{equation*}
    A = \min_{\vw \in \R^d, \; b \in \R} \; \norm{\vw^{\top} \Phi + b - f}_\px^2\leq \frac{s_1^2 - (1-\eps)^2}{s_1^2 - s_{d+1}^2}  .
\end{equation*}
When $u_1^2 = 1 - \frac{s_1^2 - (1-\eps)^2}{s_1^2 - s_{d+1}^2} ,  u_{d+1}^2 = \frac{s_1^2 - (1-\eps)^2}{s_1^2 - s_{d+1}^2} $, the equality holds. Thus, the approximation error reaches its lower bound when $\Phi$ learns the contexture of $\tp$.
\end{proof}

\section{Evaluating an Arbitrary Encoder}
\label{app:arbitrary-encoder}

Given a context that is compatible with the task, the encoder that learns the contexture is optimal. 
Now, what about an arbitrary encoder $\Phi$? Is it possible to bound its worst-case approximation error on the class of compatible tasks?
To derive such a bound, two key objects are necessary: the induced RKHS and the ratio trace.
They were originally defined in \cite{zhai2023understanding} for self-supervised learning, and here we extend them to a broader scope.

Denote the range of $\tpstar$ by $R(\tpstar) = \sset{\tpstar f}{f \in \lxp}$.

\begin{definition}
\label{def:induced-rkhs}
The \textbf{induced RKHS} of $P^+$, denoted by $\hpp$, is the Hilbert space $R(\tpstar)$ with the inner product given by $\dotp{\tpstar f_1, \tpstar f_2}_{\hpp} = \dotp{f_1, f_2}_{\px}$.
\end{definition}

An alternative formula is that for any $h_1,h_2 \in \hpp$ where $h_1 = \sum u_i \nu_i$ and $h_2 = \sum v_i \nu_i$, there is $\dotp{h_1, h_2}_{\hpp} = \sum \frac{u_i v_i}{s_i^2}$.

\begin{proposition}
\label{prop:induced-rkhs}
The induced RKHS $\hpp$ has the following properties:
\begin{enumerate}[label=(\roman*)]
    \item $\ka$ is the reproducing kernel, such that $h(a) = \dotp{h, \ka(a, \cdot)}_{\hpp}$ for all $h \in \hpp$.
    \item $\hpp$ is isometric to $\sspan \oset{\mu_i: s_i > 0}$, which is a subspace of $\lxp$.
    \item $f^* \in \fep$ is equivalent to $h^* = \tpstar f^*$ satisfying the following \textbf{isometry property}:
    \begin{equation}
    \label{eqn:isometry-property}
    (1-\epsilon) \norm{\tilde{h}^*}_{\hpp} \le \norm{\tilde{h}^*}_{\pa} \le \norm{\tilde{h}^*}_{\hpp}  .
    \end{equation}
\end{enumerate}
\end{proposition}

\begin{proof}
For any $h \in \hpp$ where $h = \tpstar f$ and $f = \sum u_i \mu_i$, we have
\begin{equation*}
    \dotp{h, \ka(a,\cdot)}_{\hpp} = \dotp{\sum s_i u_i \nu_i, \sum s_i^2 \nu_i(a) \nu_i}_{\hpp} = \sum s_i u_i \nu_i(a) = h(a),
\end{equation*}
which proves (i). (ii) is obvious. Regarding (iii), recall that $f^* = \sum u_i \mu_i \in \fep$ is equivalent to $\sum_{i \ge 1} s_i^2 u_i^2 \ge (1-\epsilon)^2 \sum_{i \ge 1} u_i^2$, and this is $\norm{\tilde{h}^*}_{\pa} \ge (1-\epsilon) \norm{\tilde{h}^*}_{\hpp}$.
Meanwhile, it is obvious that $\norm{\tilde{h}^*}_{\pa} \le \norm{\tilde{h}^*}_{\hpp}$ always holds.
\end{proof}

\begin{definition}
\label{def:ratio-trace}
    Define covariance matrices $\mC_{\Phi} = \Cov_{\px}[\Phi]$, and $\mB_{\Phi} = \Cov_{\pa}[\tpstar \Phi]$. If $\mC_{\Phi}$ is invertible, then the \textbf{ratio trace} of $\Phi$ \wrt{} $P^+$ is defined as $\rt(\Phi; P^+) = \rt(\phi_1,\cdots,\phi_d; P^+) := \Tr(\mC_{\Phi}^{-1} \mB_{\Phi})$;
    otherwise, let $\Phi' = \brac{\phi_{i_1},\cdots,\phi_{i_t}}$ be the maximal linearly independent subset of $\brac{\phi_1,\cdots,\phi_d}$, and define the ratio trace of $\Phi$ the same as the ratio trace of $\Phi'$.
\end{definition}

The ratio trace of any $\Phi$ essentially measures how well $\Phi$ is aligned with the contexture of $P^+$.
Multiplying $\Phi$ by any invertible matrix does not change its ratio trace.
If $\Phi$ learns the contexture, then its ratio trace is $s_1^2 + \cdots + s_d^2$, which can be easily shown by setting $\phi_i = \mu_i$.
In fact, this is the maximum ratio trace of any $d$-dimensional encoder.

\begin{lemma}
\label{lem:ratio-trace}
Suppose $\phi_1,\cdots,\phi_d$ are orthonormal and all have zero mean. Then, we have
\begin{equation*}
    \norm{\tpstar \phi_1}_{\pa}^2 + \cdots + \norm{\tpstar \phi_d}_{\pa}^2 \le s_1^2 + \cdots + s_d^2 .
\end{equation*}
\end{lemma}
\begin{proof}
    Let $\phi_i = \sum_{j \ge 1} q_{ij} \mu_j$ for $i \in [d]$.
    Then, $\mQ = (q_{ij})$ is a matrix with $d$ orthonormal rows and infinitely many columns.
    It is easy to see that the left-hand side is equal to $\Tr(\mQ \mD \mQ^{\top})$, where $\mD = \diag \oset{s_1^2, s_2^2, \cdots}$.
    Let $\vq_j$ be the $j$-th column of $\mQ$. For all $j \in [d]$, there is $\sum_{i=1}^j \vq_i^{\top} \vq_i \le j$; and for any $j > d$, there is $\sum_{i=1}^j \vq_i^{\top} \vq_i \le d$.
    Thus, using Abel transformation, we have
    \begin{equation*}
        \Tr(\mQ \mD \mQ^{\top}) = \Tr(\mD \mQ^{\top} \mQ) = \sum_{j=1}^{\infty} s_j^2 \vq_j^{\top} \vq_j = \sum_{j=1}^{\infty} \paren{\sum_{i=1}^j \vq_i^{\top} \vq_i } \paren{s_j^2 - s_{j+1}^2} \le \sum_{j=1}^d s_j^2 ,
    \end{equation*}
    as desired.
\end{proof}

The ratio trace induces a key quantity in the approximation error bound called the trace gap, which reflects the gap between $\Phi$ and the top-$d$ singular functions.
The larger the trace gap is, the larger the approximation error will be.
A simple definition is $s_1^2 + \cdots + s_{d+1}^2 - \rt(\Phi; P^+)$, whose lower bound $s_{d+1}^2$ can be achieved by the top-$d$ singular functions, the optimal encoder.
However, there is an issue with this definition.
For example, consider an encoder with $d = 1000$.
It learns the top-$10$ singular functions, but the other $990$ dimensions are complete noise that has zero contribution to $\rt(\Phi; P^+)$.
The approximation error of this encoder should be no higher than that of the top-$10$ singular functions, because adding more dimensions will never make the approximation error higher.
However, if $d$ becomes larger and $\rt(\Phi; P^+)$ stays the same, then $s_1^2+\cdots+s_{d+1}^2 - \rt(\Phi; P^+)$ will become larger, so this quantity does not correlate with the approximation error in this scenario.
The following definition fixes this issue.
\begin{definition}
\label{def:trace-gap}
For any linearly independent $f_1,\cdots,f_{d'} \in \lxp$, denote $F = [f_1,\cdots,f_{d'}]$, $\mC_F = \Cov_{\px}[F]$, and $\mB_F = \Cov_{\pa}[F]$.
The \textbf{trace gap} of $\Phi$ \wrt{} $P^+$ is defined as
\begin{equation*}
    \tg(\Phi; P^+) := \inf_{d' \le d} \; \inf_{f_1,\cdots,f_{d'}} \; \oset{s_1^2 + \cdots + s_{d'+1}^2 - \Tr(\mC_F^{-1} \mB_F)} .
\end{equation*}
\end{definition}

Obviously, this definition of trace gap is upper bounded by $s_1^2 + \cdots + s_{d+1}^2 - \rt(\Phi; P^+)$.
It solves the issue in the previous example because having completely noisy dimensions does not affect the trace gap.
The following result bounds the approximation error.

\begin{theorem}
\label{thm:approx-error-bound}
Suppose $\tg(\Phi; P^+) < s_1^2$, and $\epsilon > 1 - s_1$. Then,
\begin{equation*}
    \err(\Phi; \fep) \le \frac{s_1^2 - (1-\epsilon)^2 + s_1 \tg(\Phi; P^+)}{s_1^2 - \tg(\Phi; P^+)^2}  .
\end{equation*}
\end{theorem}
\begin{remark}
    This bound is fairly tight. If $\Phi$ learns the contexture, then by \Cref{thm:top-d-optimal} we have $\err(\Phi; \fep) = \frac{s_1^2 - (1-\epsilon)^2}{s_1^2 - s_{d+1}^2}$, and $\tg(\Phi; P^+) = s_{d+1}$.
    Compared to this exact formula, the above upper bound only has an extra $s_1 \tg(\Phi; P^+)$ term in the numerator.
\end{remark}
\begin{proof}
Let $f_1,\cdots,f_{d'}$ be the functions that minimize $s_1^2 + \cdots + s_{d'+1}^2 - \Tr(\mC_F^{-1} \mB_F)$. Without loss of generality, assume that $f_1,\cdots,f_{d'}$ have zero mean and are orthonormal. Let $\gF = \sspan \oset{f_1,\cdots,f_{d'}}$, and $\gH = \sspan \oset{\tpstar f_1,\cdots,\tpstar f_{d'}}$.
For any $f \in \fep$ with $\norm{f}_{\px} = 1$, let $h = \tpstar f \in \hpp$, and let $f_{\gF}$ be the projection of $f$ onto $\gF$.
Since $\err(\Phi; \fep)$ is upper bounded by $\norm{f - f_{\gF}}_{\px}^2$, it suffices to show that $\norm{f - f_{\gF}}_{\px}^2$ is upper bounded by the right-hand side.

Let $\alpha^2 = \norm{f_{\gF}}_{\px}^2$, and $\beta^2 = \norm{f - f_{\gF}}_{\px}^2$, where $\alpha$ and $\beta$ are non-negative.
Then, $\alpha^2 + \beta^2 = \norm{f}_{\px}^2 = 1 = \norm{h}_{\hpp}^2$.
The isometry property says that $(1-\epsilon)^2(\alpha^2+\beta^2) \le \norm{h}_{\pa}^2$.
Let $f - f_{\gF} = \beta f_0$ where $\norm{f_0}_{\px} = 1$.
Let $h_{\gF} = \tpstar h_{\gF}$ and $h_0 = \tpstar f_0$.
Then, we have $\norm{h_{\gF}}_{\pa}^2 \le s_1^2 \norm{f_{\gF}}_{\px}^2 = s_1^2 \alpha^2$.
Meanwhile, since $f_0$ is orthogonal to $f_1,\cdots,f_{d'}$, by \Cref{lem:ratio-trace} we have $\norm{\tpstar f_0}_{\pa}^2 + \norm{\tpstar f_1}_{\pa}^2 + \cdots + \norm{\tpstar f_{d'}}_{\pa}^2 \le s_1^2 + \cdots + s_{d'+1}^2$, which implies that $\norm{\tpstar f_0}_{\pa}^2 \le s_1^2 + \cdots + s_{d'+1}^2 - \Tr(\mC_{F}^{-1} \mB_{F}^{-1})$.
Let $\tau = \tg(\Phi; P^+)$. Then, we have
\begin{equation*}
    \norm{h}_{\pa}^2 = \norm{h_{\gF} + \beta h_0}_{\pa}^2 \le \norm{h_{\gF}}_{\pa}^2 + \beta^2 \norm{h_0}_{\pa}^2 + 2 \beta \norm{h_{\gF}}_{\pa} \norm{h_0}_{\pa} \le s_1^2 \alpha^2 + \tau^2 \beta^2 + 2 s_1 \tau \alpha \beta .
\end{equation*}
Thus, we have $(1-\epsilon)^2(\alpha^2+\beta^2) \le s_1^2\alpha^2 +\tau^2 \beta^2 + 2 s_1 \tau \alpha \beta$, which implies that $(s_1^2 - \tau^2) \beta ^2 \le [s_1^2 - (1-\epsilon)^2] (\alpha^2 + \beta^2) + 2 s_1 \tau \alpha \beta \le [s_1^2 - (1-\epsilon)^2 + s_1 \tau] (\alpha^2 + \beta^2) $, as desired.
\end{proof}

\paragraph{Connection to Fisher discriminant analysis.}
Fisher discriminant analysis \cite{mika1999fisher,baudat2000generalized,liu2004improving}, or more generally linear discriminant analysis (LDA), is a classical method of learning linear classifiers in statistics.
Here we show that Fisher discriminant analysis has a strong connection to the contexture theory.
Suppose $\gX \subseteq \R^{\dx}$.
Fisher discriminant analysis defines the following \textbf{between-class covariance matrix} $\mS_B \in \R^{\dx \times \dx}$ and \textbf{within-class covariance matrix} $\mS_W \in \R^{\dx \times \dx}$:
\begin{align*}
    \mS_B & = \iint \oset{ \paren{ \E \sbrac{X}{A=a_1} - \E \sbrac{X}{A=a_2} } \paren{  \E \sbrac{X}{A=a_1} - \E \sbrac{X}{A=a_2}}^{\top}} ; \\ 
    \mS_W & = \int \E_{P^+} \sbrac{\paren{X - \E \sbrac{X}{A=a}} \paren{X - \E \sbrac{X}{A=a}}^{\top}}{A=a} d\pa (a).
\end{align*}
In the original formulation of Fisher discriminant analysis, $A$ is the label of $X$.
Here we extend it to a general context variable.
Consider a linear encoder $\Phi(x) = \mW x$, where $\mW \in \R^{d \times \dx}$.
Then, one solves the following optimization problem to find $\mW$:
\begin{equation*}
    \underset{\mW \in \R^{d \times \dx}}{\text{maximize}} \; J(\mW) = \Tr \brac{ \paren{\mW \mS_B \mW^{\top}} \paren{\mW \mS_W \mW^{\top}}^{-1} } \quad \text{s.t.} \quad \mW \mS_W \mW^{\top} \text{ is invertible} .
\end{equation*}
Here, $J(\mW)$ is called the \textbf{Fisher discriminant}.
Define $\Psi(a) = \E_{P^+}[\mW X |A=a]$. Then, we can see that
\begin{align*}
    \mW \mS_B \mW^{\top} & = \iint \paren{ \Psi(a_1) - \Psi(a_2) } \paren{ \Psi(a_1) - \Psi(a_2) }^{\top} d \pa (a_1) d \pa (a_2) ; \\ 
    \mW \mS_W \mW^{\top} & = \int \E_{P^+} \sbrac{\paren{\Phi(X) - \Psi(a)}\paren{\Phi(X) - \Psi(a)}^{\top}}{A = a} d \pa (a) .
\end{align*}
Let $\mC_{\Phi} = \E[\tPhi(X)\tPhi(X)^{\top}]$ and $\mB_{\Phi} = \E[\tPsi(A) \tPsi(A)^{\top}]$.
Then, we have
\begin{align*}
    \mW \mS_B \mW^{\top} & = 2 \oset{ \E \brac{\Psi(A) \Psi(A)^{\top}} - \bPsi \bPsi^{\top} } = 2 \E \brac{\tPsi(A) \tPsi(A)^{\top}} = 2 \mB_{\Phi} ; \\ 
    \mW \mS_W \mW^{\top} & = \int \E_{P^+} \sbrac{\Phi(X) \Phi(X)^{\top} - \Psi(a) \Psi(a)^{\top}}{A = a} d \pa (a)  \\ 
    & = \E\brac{ \Phi(X) \Phi(X)^{\top} } - \E \brac{\Psi(A) \Psi(A)^{\top}} \\ 
    & = \E\brac{ \tPhi(X) \tPhi(X)^{\top} } - \E \brac{\tPsi(A) \tPsi(A)^{\top}} = \mC_{\Phi} - \mB_{\Phi}  .
\end{align*}
Therefore, $J(\mW) = 2 \Tr[(\mC_{\Phi} - \mB_{\Phi})^{-1} \mB_{\Phi}]$, which is very similar to the ratio trace defined in \Cref{def:ratio-trace}.
Recall that an encoder that learns the contexture maximizes the ratio trace.
A well-known result is that $J(\mW)$ is maximized when $\mW$ consists of the top-$d$ eigenvectors of $\mS_W^{-1} \mS_B$.
Hence, Fisher discriminant analysis is almost equivalent to contexture learning under the constraint that the encoder must be linear.

\section{Efficient Estimation of the Context Usefulness Metric}
\label{app:efficient-estimation}

To estimate the metric $\tau_d$ defined in~\Cref{eqn:metric}, it suffices to estimate the top-$d_0$ eigenvalues of the context. This can be efficiently done with the following procedure:
\begin{enumerate}[label=(\roman*)]
    \item Train an encoder $\Phi$ whose output dimension is at least $d_0$ to learn the contexture with a random subset of $m$ samples.
    \item Estimate the covariance matrix $\mC_\Phi \in \R^{d \times d} = \Cov_{\px}[\Phi]$ with Monte Carlo.
    \item Estimate $\mB_{\Phi} \in \R^{d \times d}$, where $\mB_{\Phi}[i,j] = \dotp{\tphi_i, \tkx \tphi_j}_{\px}$, with Monte Carlo.
    \item Solve the generalized eigenvalue problem $\mB_\Phi \vv = \lambda \mC_\Phi \vv$. The eigenvalues $\lambda_1 \ge \cdots \ge \lambda_{d_0}$ are estimates of $s_1^2,\cdots,s_{d_0}^2$.
    Moreover, let $\mQ = [\vv_1,\cdots,\vv_d]$ where $(\vv_i)$ are the orthonormal eigenvectors corresponding to $(\lambda_i)$. Let $\Phi^*$ be the normalized version of $\tPhi \mQ$ such that each dimension is scaled to unit variance. Then, $\phi_1^*,\cdots,\phi_{d_0}^*$ are estimates of $\mu_1,\cdots,\mu_{d_0}$.
\end{enumerate}

If we only need to estimate the eigenvalues,
then \cite{shawe2005eigenspectrum} showed that for any fixed $d$, the sum $s_1^2+ \cdots+s_d^2$ can be estimated with low error using $\Theta(d)$ \iid{} samples.
By union bound, all $s_1^2,\cdots,s_{d_0}^2$ can be estimated with low error using $m = \Theta (d_0 \log d_0)$ \iid{} samples.
However, if we want to estimate the eigenfunctions as well, then usually we need to use the entire training set.

Let us demonstrate this method on 3 real datasets from OpenML \cite{OpenML2013}: \texttt{abalone}, \texttt{fifa}, and \texttt{kings\_county}.
We use KNN with $K = 60$ as context, where $\gA = \gX$, and $P^+(x'|x) = K^{-1}$ if $x'$ is a $K$-nearest neighbor of $x$ and $0$ otherwise.
For this context, we can exactly compute $\kx$, and thus we can obtain the exact eigenvalues (ground truth) using kernel PCA.
Meanwhile, we pretrain $\Phi$ with one of the variational objectives using a random subset of $m$ samples, and estimate the eigenvalues using the post-hoc approach.
Then, we compare the estimation with the ground truth.

\begin{figure}[t]
    \centering
    \begin{tikzpicture}

\begin{axis}[
name=plot1,
 height = .26\linewidth,
    width = .28\linewidth,
    axis lines=left,  
    axis line style={-latex},  
    xlabel={$i$},
    ylabel={$s_i^2$},
    ylabel style={
        at={(0.13,0.8)}, 
        anchor=south,  
        rotate=270  
    },
    xlabel style = {
        at={(0.9,0)},
        anchor=south,
    },
title style={at={(0.5,-0.3)}, anchor=north,font=\small},
title={(a) abalone ($n=4177$)},
    xmin=0, xmax=256,
    ymin=0, ymax=1.05,
    xtick distance = 100,
    ytick distance = 0.5,
]
    \addplot[ultra thick,red,dotted] table [
        x=id,  
        y=exact,            
        col sep=comma       
    ] {data/spectrum_abalone.csv};
    \addplot[very thick,orange] table [
        x=id,  
        y=m300,            
        col sep=comma       
    ] {data/spectrum_abalone.csv};
    \addplot[very thick,purple] table [
        x=id,  
        y=m1000,            
        col sep=comma       
    ] {data/spectrum_abalone.csv};
    \addplot[very thick,teal] table [
        x=id,  
        y=m2000,            
        col sep=comma       
    ] {data/spectrum_abalone.csv};
    \addplot[very thick,blue] table [
        x=id,  
        y=mNone,            
        col sep=comma       
    ] {data/spectrum_abalone.csv};
\end{axis}

\begin{axis}[
name=plot2,
at={($(plot1.south east)+(40,0)$)},
 height = .26\linewidth,
    width = .28\linewidth,
    axis lines=left,  
    axis line style={-latex},  
    xlabel={$i$},
    ylabel={$s_i^2$},
    ylabel style={
        at={(0.13,0.8)}, 
        anchor=south,  
        rotate=270  
    },
    xlabel style = {
        at={(0.9,0)},
        anchor=south,
    },
title style={at={(0.5,-0.3)}, anchor=north,font=\small},
title={(b) fifa ($n=19178$)},
    xmin=0, xmax=256,
    ymin=0, ymax=1.05,
    xtick distance = 100,
    ytick distance = 0.5,
legend style = {nodes={scale=0.9, transform shape},at={(1.2,0.5)},anchor=west,inner sep=.5pt},
]
    \addplot[ultra thick,red,dotted] table [
        x=id,  
        y=exact,            
        col sep=comma       
    ] {data/spectrum_fifa.csv};
    \addplot[very thick,orange] table [
        x=id,  
        y=m300,            
        col sep=comma       
    ] {data/spectrum_fifa.csv};
    \addplot[very thick,purple] table [
        x=id,  
        y=m1000,            
        col sep=comma       
    ] {data/spectrum_fifa.csv};
    \addplot[very thick,teal] table [
        x=id,  
        y=m2000,            
        col sep=comma       
    ] {data/spectrum_fifa.csv};
    \addplot[very thick,blue] table [
        x=id,  
        y=mNone,            
        col sep=comma       
    ] {data/spectrum_fifa.csv};
\end{axis}

\begin{axis}[
name=plot3,
at={($(plot2.south east)+(40,0)$)},
 height = .26\linewidth,
    width = .28\linewidth,
    axis lines=left,  
    axis line style={-latex},  
    xlabel={$i$},
    ylabel={$s_i^2$},
    ylabel style={
        at={(0.13,0.8)}, 
        anchor=south,  
        rotate=270  
    },
    xlabel style = {
        at={(0.9,0)},
        anchor=south,
    },
title style={at={(0.5,-0.3)}, anchor=north,font=\small},
title={(c) kings\_county ($n=21613$)},
    xmin=0, xmax=256,
    ymin=0, ymax=1.05,
    xtick distance = 100,
    ytick distance = 0.5,
legend style = {nodes={scale=0.9, transform shape},at={(1.2,0.5)},anchor=west,inner sep=.5pt},
]
    \addplot[ultra thick,red,dotted] table [
        x=id,  
        y=exact,            
        col sep=comma       
    ] {data/spectrum_kings_county.csv};
    \addplot[very thick,orange] table [
        x=id,  
        y=m300,            
        col sep=comma       
    ] {data/spectrum_kings_county.csv};
    \addplot[very thick,purple] table [
        x=id,  
        y=m1000,            
        col sep=comma       
    ] {data/spectrum_kings_county.csv};
    \addplot[very thick,teal] table [
        x=id,  
        y=m2000,            
        col sep=comma       
    ] {data/spectrum_kings_county.csv};
    \addplot[very thick,blue] table [
        x=id,  
        y=mNone,            
        col sep=comma       
    ] {data/spectrum_kings_county.csv};
    \addlegendentry{Ground truth}
    \addlegendentry{$m=300$}
    \addlegendentry{$m=1000$}
    \addlegendentry{$m=2000$}
    \addlegendentry{$m=n$}
\end{axis}
\end{tikzpicture}
    \caption{Estimating the eigenvalues using a random subset of $m$ samples, from $n$ total training samples.}
    \label{fig:spectrum-estimation}
\end{figure}
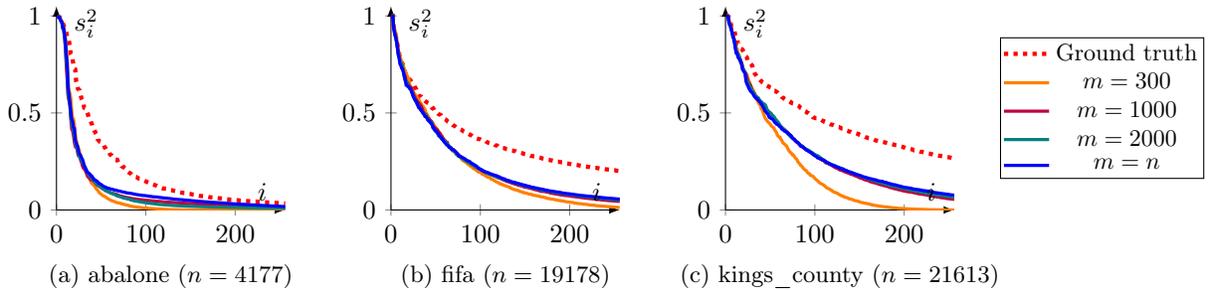

We use a 2-layer wide Tanh-activated neural network with embedding dimension $d=512$ and hidden dimension 20,000 as $\Phi$.
We train the model through non-contrastive learning with the orthonormality constraint implemented by VICReg \cite{bardes2021vicreg}, and AdamW \cite{DBLP:journals/corr/KingmaB14,loshchilov2017decoupled} as the optimizer.
We vary $m$ and compare the estimated top-$d_0$ eigenvalues with the ground truth, where $d_0 = 256$.
The estimated eigenvalues and the ground truth are plotted in \Cref{fig:spectrum-estimation}.
From the plots, we observe that the eigenvalues estimated by our estimation method decay faster than the ground truth, even if the full dataset is used.
We hypothesize that the main reason is that even though we use a very wide neural network, its function class is still a subset of $\lxp$.
Consequently, the inductive bias of the model architecture has an impact on the encoder, and therefore the learned contexture can be viewed as a mixture of the inductive bias and the original KNN context.
This mixture causes the eigenvalues to decay faster, which explains the observation in \Cref{fig:spectrum-estimation}.
Another reason is related to optimization. Since the model is non-convex, gradient methods cannot find the minima of the objective.

The average estimation error of the top-$256$ eigenvalues is reported in \Cref{tab:spectrum-est}.
The error is defined as $\frac{1}{d_0} \sum_{i=1}^{d_0} |\hat{\lambda}_i - s_i^2|$, where $\hat{\lambda}_i$ is the estimated eigenvalue.
The table shows that when $m \in [600,1000] \approx [0.5 d_0 \log d_0, 0.7 d_0 \log d_0]$, the performance is comparable to using the full dataset, which verifies the theoretical result of \cite{shawe2005eigenspectrum}.
The estimation error is not zero even if the full dataset is used due to the aforementioned reasons.
In summary, the post-hoc method can estimate the eigenvalues using a small subset of samples, but the estimated eigenvalues decay faster than the ground truth.

\begin{table}[t]
    \centering
    \begin{tabular}{l|llllll}
    \toprule
    \textbf{Dataset} & $m=100$ & $m=300$ & $m=600$ & $m=1000$ & $m=2000$ & Full dataset \\ 
    \midrule 
    \texttt{abalone} & 0.157 & 0.124 & 0.088 & 0.104 & 0.110 & 0.088 \\ 
    \texttt{fifa} & 0.218 & 0.151 & 0.137 & 0.134 & 0.133 & 0.131 \\ 
    \texttt{kings\_county} & 0.278 & 0.264 & 0.190 & 0.183 & 0.177 & 0.177 \\ 
    \bottomrule
    \end{tabular}
    \caption{Average estimation error of the top-$256$ eigenvalues.}
    \label{tab:spectrum-est}
\end{table}

\section{Scaling Law Experiment Details}
\label{app:scaling-law-experiment}

Here we provide a more detailed description of the experiment setting in \Cref{sec:scaling-law}.

\paragraph{Experiment overview.}
The purpose of this experiment is to examine whether a large neural network can learn the contexture well, and whether scaling up the model size makes the learned representation more aligned to the top-$d$ eigenfunctions.
We compare two encoders.
The first encoder is obtained via kernel PCA on the dual kernel, so it consists of the exact top-$d$ eigenfunctions.
The second encoder is obtained via training a large neural network to optimize an objective that can learn the contexture.
Then, we compute the representational alignment of these two encoders.
The most classical metric is the canonical-correlation analysis (CCA) metric $R_{\textrm{CCA}}^2$, which is invariant under invertible linear transformations to the encoders.
\cite{pmlr-v97-kornblith19a} proposed a variant called linear CKA, which is only invariant under orthogonal transformations.
In our setting, since we only care about the span of $\phi_1,\cdots,\phi_d$, we would like the metric to be invariant under all invertible transformations, which is why we use CCA.
In addition, we also use the mutual KNN metric with 10 neighbors proposed by \cite{pmlr-v235-huh24a}, which measures the intersection over union (IoU) of nearest neighbors between the two representations.
This metric is not invariant under invertible linear transformations,
so we whiten the two representations such that their covariance matrices are both identities.

\paragraph{Setup.}
We use the \texttt{abalone} dataset from OpenML, and split the dataset into a pretrain set, a downstream train set and a downstream test set by 70\%-15\%-15\%.
We use K-nearest neighbors (KNN) with $K = 30$ as the context.
The embedding dimension is set to $d = 128$.
For the second encoder, we train a fully-connected neural network with Tanh activation and skip connections for a sufficient number of steps with full-batch AdamW, and vary the depth and width of the network so that we can study their effect on the alignment.
Here, ``depth'' refers to the number of hidden layers---for example, a 2-layer neural network has depth 1.
For each width and depth, we run the experiments 15 times with different random initializations and report the average alignment.

In our experiments, we observe the \textbf{dimension collapse} problem \cite{jing2022understanding}---if we set the output dimension of the neural network to be $d$, then the rank of the learned representation will usually be less than $d$, meaning that it can only extract the top-$d'$ eigenspace for some $d' < d$.
\cite{jing2022understanding} proved that the training dynamics of self-supervised learning can cause this problem, that is, a large neural network trained with a gradient method cannot find the exact minima, but will find a low-rank solution instead.

To fix this issue, we set the output dimension of the neural network to be $d_1 = 512 > d$.
After we obtain the $d_1$-dimensional encoder, similar to \Cref{app:efficient-estimation} we estimate the matrices $\mC_\Phi$ and $\mB_\Phi$, and solve the generalized eigenvalue problem $\mB_\Phi \vv = \lambda \mC_\Phi \vv$.
Let $\mV = [\vv_1,\cdots,\vv_d] \in \R^{d_1 \times d}$ be the top-$d$ eigenvectors;
then, we use $\tPhi \mV$ as the $d$-dimensional representation.
In other words, we use the 128 principal components of the 512-dimensional embedding.

\section{More on Context Evaluation}
\label{app:context_evaluation}
In this section, we offer guidance for practitioners on identifying contexts with weak or strong associations with inputs. We then show that both excessively weak and overly strong associations degrade downstream performance and demonstrate that the proposed quantitative measurements accurately capture association strength in our controlled experiments. Moreover, we provide full experimental results that complements \Cref{tab:correlation}.
Finally, we provide proofs for the lemmas in Section~\ref{sec:association}.

\subsection{Quantitative measurements for level of association}
\label{app:measurement_for_association}
While mutual information captures mutual dependence between random variables, estimating it from samples remains a long-standing challenge as it requires the joint density function to be known~\cite{paninski2003estimation}. To address this, we propose alternative metrics that are computationally more tractable.

\begin{itemize}
    \item \textbf{Decay rate (all association):} As shown in the top row of \Cref{fig:taud}, the decay rate of singular values $(s_i)_{i\geq 0}$ reflects the strength of association. To estimate the decay rate $\lambda$, we assume the singular values decay exponentially and fit the regression model $s_i^2 = \exp{(-\lambda i)}$. When $\lambda$ is large, it indicates a fast decay rate and context has low association. Conversely, when $\lambda$ is low, it implies a slow decay and highly associated context. 
    
    \item \textbf{Expected kernel deviation (weak association):} Since the kernel values $\ka(a,a')$ can be close to $1$ for the contexts with low association, we propose using the expected absolute deviation from 1, i.e. $\mathbb{E}_{x,x' \sim \px} \left[ |\kx(x,x')-1| \right]$, as a measure to indicate the weak association setting. Given samples $x_1,\cdots, x_n \in \gX$, we use Monte Carlo sampling to approximate it, i.e. 
    \[
    \frac{1}{n^2}\sum_{i=1}^n \sum_{j=1}^n |\kx(x_i,x_j)-1|.
    \]
    
    \item \textbf{Lipschitz constant (strong association):} We empirically measure the Lipschitz constant of $\kx$ for detecting contexts with strong association. Specifically, given samples $x_1,\cdots, x_n \in \gX$, we use the following estimation:
    \[
    L_\kx = \sup_{x, y, z\in \gX} \frac{|\kx(z,x) - \kx(z, y)|}{ ||x-y||_2}
    \approx \max_{ 1 \leq i < j \leq n, 1 \leq k \leq n} \frac{|\kx(x_k,x_i) - \kx(x_k, x_j)|}{ ||x_i-x_j||_2}.
    \]
    
\end{itemize}

We note that the first metric requires estimating singular values $(s_i^2)_{i\geq 0}$ and the last two metrics rely on estimating kernel values $\kx(x,x')$. For estimating singular values, we employ the same technique as the task agnostic metric $\tau$ in \Cref{eqn:metric}. Estimating kernel values, on the other hand, necessitates training an encoder $\Phi$ and approximating $\kx(x,x')$ as $ \Phi(x)^\top \Phi(x')$, which may require a large training set. Thus, decay rate measurements are preferred due to their simpler estimation process.

Additionally, for the non-smooth kernel $\kx$, we have the following lemma showing that we may have the singular functions could be non-smooth and difficult to estimate.
\begin{lemma}[Proof in \Cref{app:proof-lemma-strong_association}]
\label{lemma:strong_association}
Let the Lipschitz constant for the positive pair kernel $\kx$ be $L_\kx = \sup_{x_1,x_2,x_3 \in \gX} \frac{|\kx(x_3,x_1) - \kx(x_3, x_2)|}{ ||x_1-x_2||_2}$ and the maximum Lipschitz constant of its eigenfunctions be $L_\mu = \max_{i \geq 1} \sup_{x_1, x_2 \in \gX} \frac{|\mu_i(x_1)-\mu_i(x_2)|}{||x_1-x_2||_2}$. Assume that all the eigenfunctions $\mu_i$ are bounded by $c$, i.e. $\mu_i(x) \leq c$ for all $i>0$ and $x \in \gX$. Then we have $L_\kx \leq c L_\mu \sum_{i=1} s_i^2$.
\end{lemma}
Assume that there exists a universal $c$ that bounds all eigenfunctions. For a highly non-smooth kernel $\kx$ with high Lipschiz constant $L_{\kx}$, the lemma implies that we have either (i) smooth singular functions with large $L_\mu$ and slow decay in singular values with small $\sum_{i=1} s_i^2$, or (ii) non-smooth singular functions with large $L_\mu$ and fast decay in singular values with small $\sum_{i=1} s_i^2$.
For (i), we need a larger $d$ to approximate the kernel well, which leads to a higher downstream sample complexity.
For (ii), the function approximation by neural networks becomes more difficult for non-smooth functions~\cite{yarotsky2018optimal}. 

\subsubsection{Empirical Verification}
\paragraph{Setup.} We provide empirical evidence showing (1) downstream performance is worse for contexts with weak and strong associations, (2) the proposed quantitative measurements are positively correlated with the association level. To control the level of association, we use RBF kernels and KNN. 

For the estimation of kernel value $\kx$, we use $\kx(x,x') = \int \frac{P^+(a|x) P^+(a|x')}{\pa(a)} da$ since $P^+(a|x)$ can be efficiently computed for these contexts. The decay rate is estimated using a non-linear least squares approach to fit the regression model. For computing the expected kernel deviation, we utilize the entire training set. To estimate the Lipschitz constant, we restrict the sample size to $n=1000$ for computational efficiency.
Other experimental setups are the same as in \Cref{sec:empirical_verification}. 

\paragraph{Results.} 
\Cref{fig:association_knn_err_decay_rate} and \Cref{fig:association_rbf_err_decay_rate} illustrate the relationship between association level and both the linear probe error $\err_{d^*}$ and the decay rate $\lambda$ for KNN and RBF contexts, respectively. The results show that $\err_{d^*}$ increases at both extremes, with most blue curves exhibiting a U-shape. This suggests that both weak and strong association levels lead to higher errors. Additionally, the red curve indicates that the decay rate $\lambda$ increases as the association level strengthens, highlighting a strong correlation between association level and spectral decay, which is effectively captured by the estimated decay rate.

We report the relationship between association level and both the expected kernel deviation and the Lipschitz constant $L_{\kx}$ in \Cref{fig:association_knn_lipschitz_kernel_deviation} for the KNN context and \Cref{fig:association_rbf_lipschitz_kernel_deviation} for the RBF context. The results show that contexts with low association exhibit small expected kernel deviations, while those with high association have large Lipschitz constants. These findings align with our theoretical developments in \Cref{sec:association}.

\subsection{Proof of Lemma \ref{lemma:weak_association}}
\label{app:proof-lemma-weak_association}

\begin{proof}
Define $\kx'(x,x') = \kx(x,x')-1 = \sum_{i>0} s_i^2 \mu_i(x)\mu_i(x')$ that is the positive pair kernel without the trivial mode $(s_0,\mu_0)$, where the equality follows the definition of $T_{\kx}$. We also denote the corresponding kernel integral operator as $(T_{\kx'}g)(x) = \int g(x')\kx'(x,x') d\px(x')$. Then we have 
\begin{align*}
\sum_{i>0} s_i^2 
= \Tr(T_{\kx'})  
= \int \kx'(x,x') d \px (x') 
< \int \epsilon d \px (x') 
= \epsilon.
\end{align*}
\end{proof}

\subsection{Proof of Lemma \ref{lemma:strong_association}}
\label{app:proof-lemma-strong_association}
\begin{proof}

\begin{align*}
L_\kx
& = \sup_{x_1,x_2,x_3 \in \gX} \frac{|\kx(x_3,x_1) - \kx(x_3, x_2)|}{ \norm{x_1-x_2}_2} \\
& = \sup_{x_1,x_2,x_3 \in \gX} \frac{|\sum_{i \geq 0} s_i^2 \mu_i(x_3)\mu_i(x_1) - \sum_{i \geq 0} s_i^2 \mu_i(x_3)\mu_i(x_2)|}{ \norm{x_1-x_2}_2} \\
& = \sup_{x_1,x_2,x_3 \in \gX} \frac{|\sum_{i \geq 0} s_i^2 \mu_i(x_3) (\mu_i(x_1) - \mu_i(x_2)|}{ \norm{x_1-x_2}_2} \\
& = \sup_{x_1,x_2,x_3 \in \gX} \frac{|\sum_{i > 0} s_i^2 \mu_i(x_3) (\mu_i(x_1) - \mu_i(x_2)|}{ \norm{x_1-x_2}_2} \ \ (\mu_0 \equiv 1) \\
& \leq \sup_{x_1,x_2,x_3 \in \gX} \frac{\sum_{i > 0} s_i^2 \mu_i(x_3) |\mu_i(x_1) - \mu_i(x_2)|}{ \norm{x_1-x_2}_2} \\
& \leq \sup_{x_3 \in \gX} \sum_{i > 0} s_i^2 \mu_i(x_3) \sup_{x_1,x_2 \in \gX} \frac{ |\mu_i(x_1) - \mu_i(x_2)|}{ \norm{x_1-x_2}_2}  \\
& \leq \sup_{x_3 \in \gX} \sum_{i > 0} s_i^2 \mu_i(x_3) L_\mu  \\
& \leq \sup_{x_3 \in \gX} c L_\mu \sum_{i > 0} s_i^2.
\end{align*}

\end{proof}

\begin{figure}[!t]
    \centering
    \includegraphics[width=\textwidth]{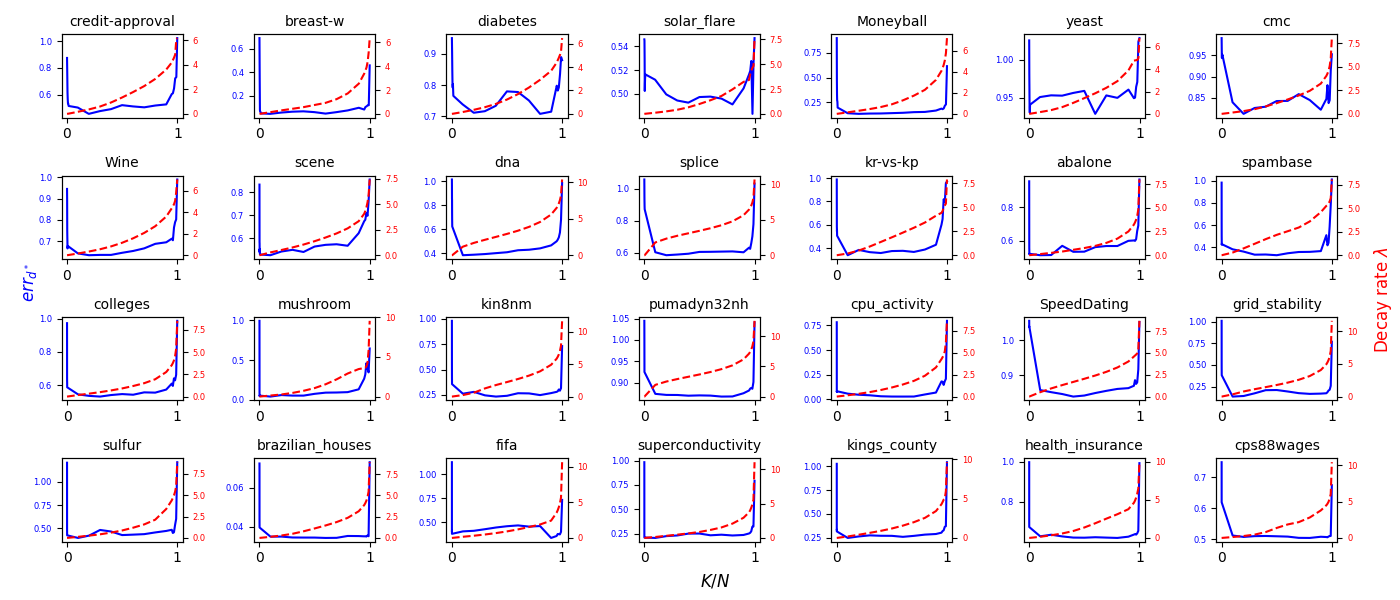}
    \caption{Association level vs $\err_{d^*}$ and the decay rate $\lambda$ for the KNN context on the 28 datasets. A larger $K/N$ indicates a lower association level, while a smaller $K/N$ corresponds to a higher association level. }
    \label{fig:association_knn_err_decay_rate}
\end{figure}

\begin{figure}[!t]
    \centering
    \includegraphics[width=\textwidth]{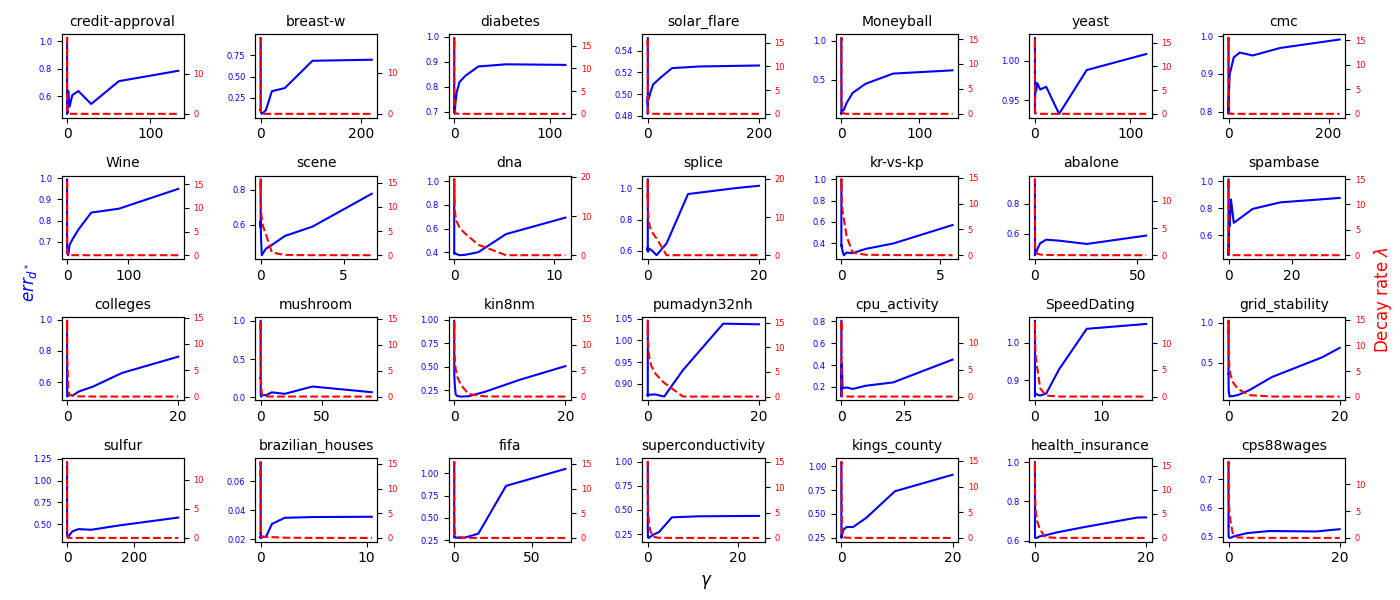}
    \caption{Association level vs $\err_{d^*}$ and the decay rate $\lambda$ for the RBF context on the 28 datasets. A larger $\gamma$ indicates a higher association level, while a smaller $\gamma$ corresponds to a higher association level.}
    \label{fig:association_rbf_err_decay_rate}
\end{figure}

\begin{figure}[!t]
    \centering
    \includegraphics[width=\textwidth]{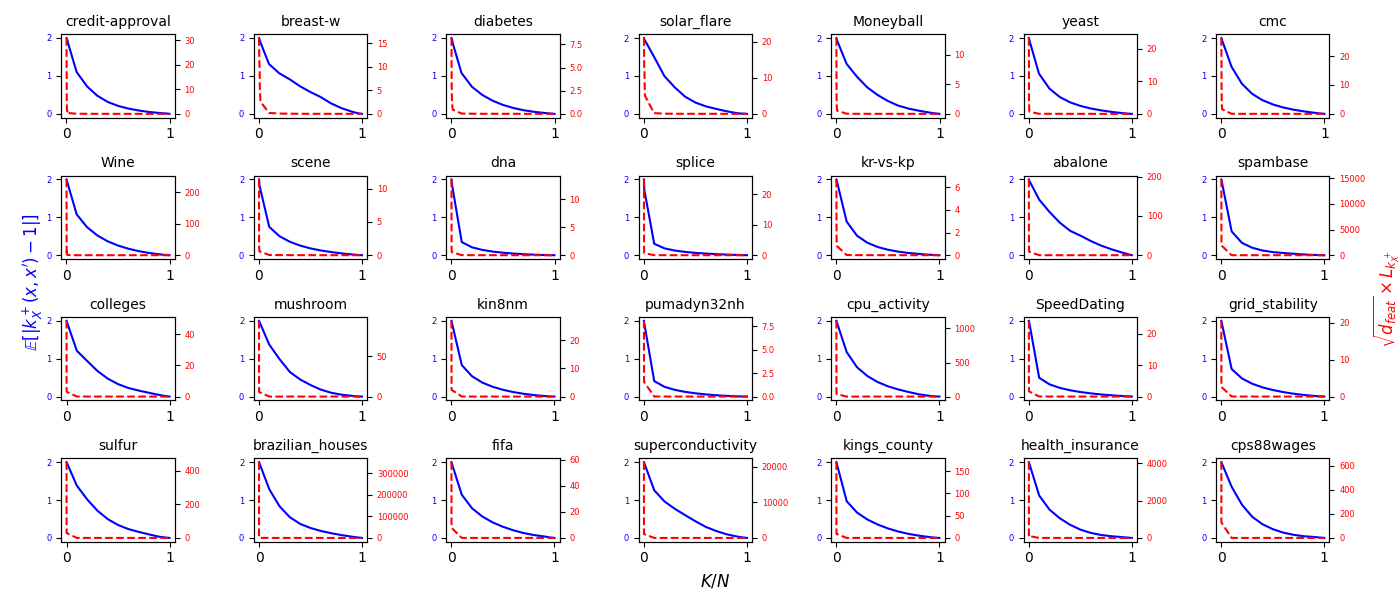}
    \caption{Association level vs the kernel deviation $\mathbb{E}_{x,x' \sim \px}|\kx (x,x')-1|$ and the Lipschitz constant $L_{\kx}$ for the KNN context on the 28 datasets. A larger $K/N$ indicates a lower association level, while a smaller $K/N$ corresponds to a higher association level. We multiply $L_{\kx}$ by the input feature dimension $d_{feat}$ to normalize the $L_2$ distance in the denominator.  }
    \label{fig:association_knn_lipschitz_kernel_deviation}
\end{figure}

\begin{figure}[!t]
    \centering
    \includegraphics[width=\textwidth]{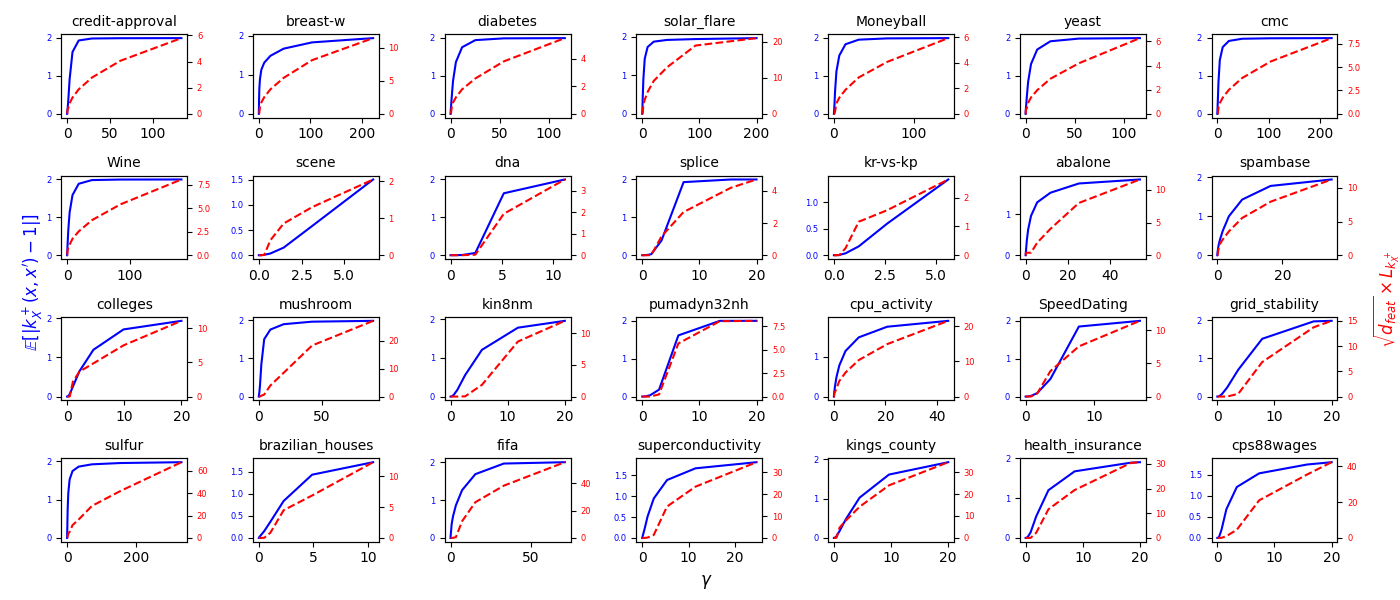}
    \caption{Association level vs the kernel deviation $\mathbb{E}_{x,x' \sim \px}|\kx (x,x')-1|$ and the Lipschitz constant $L_{\kx}$ for the RBF context on the 28 datasets.  A larger $\gamma$ indicates a higher association level, while a smaller $\gamma$ corresponds to a higher association level. We multiply $L_{\kx}$ by the input feature dimension $d_{feat}$ to normalize the $L_2$ distance in the denominator.}
    \label{fig:association_rbf_lipschitz_kernel_deviation}
\end{figure}

\end{document}